\documentclass{article}
\usepackage[utf8]{inputenc}
\usepackage[margin=1in]{geometry}
\usepackage[numbers,sort]{natbib}
\usepackage{csvsimple,booktabs,adjustbox}
\usepackage[normalem]{ulem}
\usepackage{float}
\usepackage{macros}
\usepackage{pifont}
\usepackage{enumitem}
\newcommand\contributionNote[1]{%
  \begingroup
  \renewcommand\thefootnote{}\footnote{\kern-5pt \textcolor{white}{\rule{5pt}{2ex}}#1}%
  \addtocounter{footnote}{-1}%
  \endgroup
}

\usepackage{graphicx} 

\begin{document}
\begin{center}
    \vspace*{0.5cm}

 \LARGE{Linear Recursive Feature Machines  provably \\recover low-rank matrices}  
\vspace*{0.5cm}

\large{Adityanarayanan Radhakrishnan$^{1, 2}$ \\ Mikhail Belkin$^{3}$ \\ Dmitriy Drusvyatskiy$^{4}$}

\vspace*{0.5cm}

\normalsize{$^{1}$Harvard University.} \\
\normalsize{$^{2}$Broad Institute of MIT and Harvard.} \\
\normalsize{$^{3}$Hal\i c\i o\u glu Data Science Institute, UC San Diego.} \\
\normalsize{$^{4}$University of Washington.}\\
\vspace*{0.5cm}
\end{center}

\setcounter{footnote}{3}
\begin{abstract} 
A fundamental problem in machine learning is to understand how neural networks  make accurate predictions, while seemingly bypassing the curse of dimensionality.  A possible explanation is that  common training algorithms for neural networks {\em implicitly} perform dimensionality reduction---a process called feature learning.  Recent work~\cite{radhakrishnan2022feature} posited that the effects of feature learning can be elicited from a classical statistical estimator called the average gradient outer product (AGOP). The authors proposed Recursive Feature Machines (RFMs)  as an algorithm that {\em explicitly} performs feature learning by alternating between  $(1)$ reweighting the feature vectors by the AGOP and (2) learning the prediction function in the transformed space. In this work, we develop the first theoretical guarantees for how RFM performs dimensionality reduction by focusing on the  class of overparametrized problems arising in sparse linear regression and low-rank matrix recovery. Specifically, we show that RFM restricted to linear models (\textit{lin-RFM})  generalizes the well-studied Iteratively Reweighted Least Squares (IRLS) algorithm.  Our results shed light on the connection between feature learning in neural networks and classical sparse recovery algorithms. In addition, we provide an implementation of lin-RFM that scales to matrices with millions of missing entries.  Our implementation is faster than the standard IRLS algorithm as it is SVD-free.  It also outperforms deep linear networks for sparse linear regression and low-rank matrix completion.  
\end{abstract}

\section{Introduction}

Dramatic recent successes of deep neural networks appear to have overcome the curse of dimensionality on a wide variety of learning tasks. For example, predicting the next word (token) using context lengths of $100,000$ or more should have been impossible even with very large (but still limited) modern training data and computation.  Nevertheless, recent deep learning approaches, including transformer-based large language models, succeed at this task~\cite{claude2, gpt4turbo}.  How are neural networks able to make accurate predictions, while seemingly bypassing the curse of dimensionality, and what properties of real-world datasets make this possible?

The purpose of this paper is to establish  a close connection between an implicit mechanism for dimensionality reduction in neural networks and sparse recovery, which  
 has historically been one of the most influential approaches for coping with the curse of dimensionality~\cite{wright2022high, donoho2006compressed}.  Recently, it has become clear that neural networks can  learn task-relevant low-dimensional structure when using gradient-based training algorithms---a property that is called \textit{feature learning} in the literature~\cite{YangFeatureLearning, FeatureLearningEmergenceShi,DLSRepresentationReLU}.  For example,  neural networks can provably recover a latent low-dimensional index space of multi-index models~\cite{DLSRepresentationReLU, parkinson2023linear}.  Furthermore, deep {\it linear} neural networks can find solutions of classical sparse recovery problems, including sparse linear regression and low rank matrix sensing/completion~\cite{LinearDiagonalNetKernelRichLazy, DeepMatrixFactorization}. Understanding the general principles driving feature learning in neural networks is an active area of research.

\begin{figure}[!t]
    \centering
    \includegraphics[width=\textwidth]{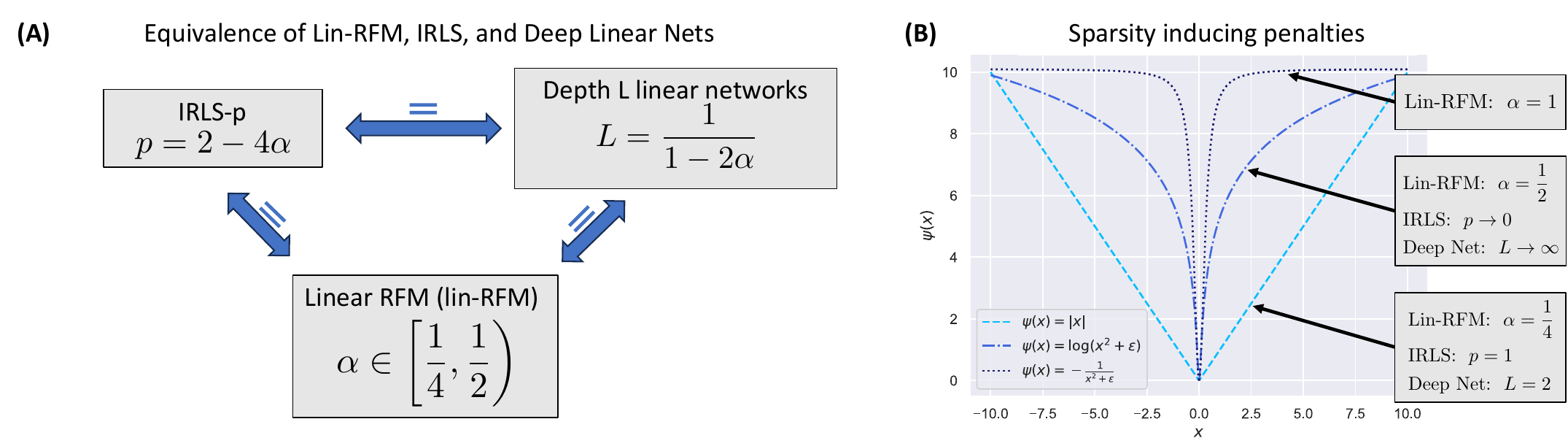}
    \caption{Overview of our results. The functions in \textbf{(B)} are normalized so that $\psi(0) = 0$.}
    \label{fig: Overiew}
\end{figure}

The recent work~\cite{radhakrishnan2022feature} proposed a simple mechanism for how neural networks learn features through training. Namely, feature learning can be elicited from a classical statistical operator called the {\em average gradient outer product (AGOP)}.  The AGOP of a function $f\colon \mathbb{R}^{d} \to \mathbb{R}$ with respect to a set of data points $x_1,\ldots, x_n$ is defined to be the $d\times d$ matrix:  
$
\text{AGOP}(f(x), (x_1, \ldots x_n))  := \frac{1}{n} \sum_{i=1}^n \nabla f(x_i) \nabla f(x_i)^T.
$
Equivalently, $\text{AGOP}(f)$ is the covariance matrix of the gradient $\nabla f(x)$ with respect to the training data. Therefore, AGOP measures the variations of the predictor $f$ and can be used to amplify those directions that are most relevant for learning the output of $f$ or filter out irrelevant directions.

We will use an important modification of AGOP, introduced in \cite{beaglehole2023mechanism}, that is better suited for structured data. Namely, it is often the case that the data $x$ decomposes into blocks that share the same salient features.  For example, a $d_1 \times d_2$ data matrix $x$ can be represented as a vector of its $d_1$ rows. If only the projection of the rows onto a low-dimensional subspace is relevant for the learning task at hand---the case for low-rank matrix recovery---then it is reasonable to seek a single dimensionality reduction mechanism that is shared across all rows.  A similar situation occurs when classifying images using patches, as is commonly done with convolutional neural networks~\cite{beaglehole2023mechanism}.  With these examples in mind, suppose that each training data point $x$ lies in a product space $\mathcal{X}:=\underbrace{\mathbb{R}^{k}\times \ldots\times \mathbb{R}^k}_{m~\text{times}}$. In the matrix example, the factors of the product $\mathcal{X}$ correspond to rows, and for convolutional neural networks, the factors correspond to image patches.  Given a function $f$ defined on $\mathcal{X}$ and data $\{x_i\}_{i=1}^{n} \subset \mathcal{X}$, we may define the AGOP of $f$ as the average of the gradient covariances along the factors: 
\begin{align}
\label{eq: Generalization of AGOP}
\text{AGOP}(f(x), (x_1, \ldots x_n)) := \frac{1}{nm} \sum_{i=1}^{n} \sum_{\ell=1}^{m} \nabla_{\ell} f(x_i) \nabla_{\ell} f(x_i)^T~.
\end{align}
Here, $\nabla_{\ell} f$ is the gradient of $f$ with respect to the $\ell$\textsuperscript{th} factor of the product space $\mathcal{X}$.  This notion of structured AGOP was first introduced in~\cite{beaglehole2023mechanism} in the context of convolutional neural networks, and it will play a central role in our work.  

Treating AGOP as a key ingredient of feature learning, the papers \cite{beaglehole2023mechanism,radhakrishnan2022feature} introduced a new algorithm for explicitly learning features with any base learning algorithm called {\em Recursive Feature Machine (RFM)}. Assume we have a base learning algorithm that produces a predictor $\hat{f}(x)$ given a training set $\{(x_i, y_i)\}_{i=1}^{n}$. Given  an initial filtering matrix $M_0 \in \mathbb{R}^{k \times k}$, RFM iteratively generates (1) a predictor $\hat{f}_t(x)$ and (2) an updated filtering matrix, $M_{t+1}$, according to the following rules: 
\begin{equation}\label{eqn:RFM}
\begin{aligned}
    &\textit{Step 1:} \qquad \text{Obtain a predictor $\hat{f}_t(x)$ by training on filtered data $\{(M_t x_i, y_i)\}_{i=1}^{n}$}~;\\
    &\textit{Step 2:} \qquad \text{Let $M_{t+1} = \phi \left( \text{AGOP($\hat{f}_t(M_tx)$, $(x_1, \ldots x_n)$)} \right)$}.\\
\end{aligned}
\end{equation}
Here, $M_t$ acts independently on all factors $x^{(j)}$ of $x \in \mathcal{X}$, i.e., $M_tx := (M_t x^{(1)}, \ldots, M_t x^{(m)})$ and $\phi\colon \mathbb{R} \to \mathbb{R}$ is any function, which, by convention, acts on symmetric matrices through their eigenvalues. Note that Step 2 always computes AGOP of each predictor $\hat{f}_t$ with respect to the original data.  The work~\cite{radhakrishnan2022feature} demonstrated state-of-the-art performance of RFM for prediction tasks on tabular datasets, when the base algorithm corresponded to kernel machines and $\phi$ was a power function.  Subsequently,  the work~\cite{beaglehole2023mechanism} showed that when the base algorithm corresponded to convolutional kernel machines and $\phi$ was a power function, RFM provided a consistent improvement over convolutional kernels on image classification tasks. Key to this improvement in performance was the fact that AGOP identified low dimensional subspaces relevant for prediction.  Nevertheless, while the RFM framework is simpler to implement and analyze than training algorithms for neural networks, there has been no precise theoretical analysis for the  performance of RFM, nor an exact characterization for the emergence of low dimensional structure through AGOP.

In this work, we develop the first theoretical guarantees for how RFM performs dimensionality reduction by focusing on the classical settings of sparse linear regression and low rank matrix recovery. Setting the stage, consider a ground truth matrix $Y_{\sharp}\in\mathbb{R}^{d_1\times d_2}$ with rank $r$ that is much smaller than $\min\{d_1,d_2\}$. The goal of low-rank matrix recovery is to recover $Y_{\sharp}$ from linear measurements
$$y_i:=\langle A_i, Y_{\sharp}\rangle\qquad \textrm{for}\qquad \forall i \in [n].$$
Here, $A_i$ are called the sensing matrices and $\langle X,Y\rangle=\Tr(X^\top Y)$ denotes the standard trace inner product. In the context of learning, we consider the pairs $\{(A_i,y_i)\}_{i=1}^n$ as the training data. Notice that the measurement $\langle A_i, Y_{\sharp}\rangle$ remains unchanged if the rows of $A_i$ are projected onto the row space of $Y_{\sharp}$. Therefore we may think of prediction functions $f(A)$ taking the rows of $A$ as arguments, i.e. operating on $d_1$-fold product of $\mathbb{R}^{d_2}$.
We will see that a simple instantiation of RFM with linear predictors proceeds according to the  rule:
\begin{equation}\label{eqn:mat_RFM}
\begin{aligned}
     \textit{(Linear Regression):} \quad W_t &= \argmin_{W \in \mathbb{R}^{d_1 \times d_2}}~ \|W\|^2_F  ~~ ~\text{subject to ~~$\langle  A_i M_t, W   \rangle  = y_i$~~ $\forall i \in [n]$}~; \\
     \textit{(AGOP Update):} \quad M_{t+1} &= \phi \left( M_t^T W_t^T W_t M_t \right)~.
\end{aligned}\tag{{lin-RFM}}
\end{equation}
The problem of sparse linear regression corresponds to the special case where all the matrices are diagonal; \ref{eqn:mat_RFM} applies in this setting without any modification.  Our contributions are as follows.

\begin{enumerate}
\item {\bf (Explicit characterization of fixed points of lin-RFM)}
We show that the fixed points of \ref{eqn:mat_RFM} are precisely the first-order critical points of the problem  
\begin{align}
    \label{eq: General lin-RFM objective}
        &\text{$\argmin_{Z \in \mathbb{R}^{d_1 \times d_2}} ~\sum_i \psi(\sigma_i(Y))$ \qquad subject to \quad$\langle A_i, Y\rangle = y_i$ \quad $\forall i \in [n]$,}
    \end{align}    
    where $\sigma_i(Y)$ are the singular values of $Y$ and $\psi(t)$ is explicitly given as $\psi(s)=\int_{0}^s\frac{r~dr}{\phi(r^2)^2}$, as long $\phi$ is continuous and takes only positive values.  In the  case of power functions $\phi(s)=(s+ \epsilon)^{\alpha}$ with $\epsilon,\alpha>0$, the equality holds:
  $$  \begin{array}{cc}
 \psi(s)=  \left\{ 
    \begin{array}{cc}
      \frac{1}{2-4\alpha} (s^2+\epsilon)^{1-2\alpha} & {\rm if}~\alpha\neq \frac{1}{2}\\
      \frac{1}{2}\log(s+ \epsilon) & {\rm if}~\alpha=\frac{1}{2}
    \end{array}\right.
    \end{array}.$$
    Here, the constant offset $\epsilon>0$ can be arbitrarily small.  Nonzero $\epsilon$ is needed for our theoretical analysis.  We note, however, that in our experiments on the matrix completion problem, the performance of RFM remains unchanged when setting $\epsilon=0$ for any value of $\alpha$.    Summarizing, RFM spans different sparse objectives including logarithmic, power functions, or negative power functions of the singular values, with larger values of $\alpha$ more aggressively promoting $Y$ to have low-rank. See Figure~\ref{fig: Overiew} for an illustration. An interesting consequence is that for $\alpha \in \left[\frac{1}{4}, \frac{1}{2}\right)$,  lin-RFM and deep linear neural networks are equivalent in the following sense: the fixed points of lin-RFM correspond to minimizing the $\ell_p$ (pseudo) norm of singular values for $p \in (0, 1]$, which is exactly the implicit bias of deep linear neural networks with depth $L = \frac{1}{1-2\alpha}$~\cite{gunasekar2018implicitconv, shang2020unified}.     

\item{{\bf (Neural Feature Ansatz (NFA) and deep lin-RFM)}} The two papers~\cite{radhakrishnan2022feature, beaglehole2023mechanism} empirically observed that the weights of a trained neural network align with AGOP of the neural network on a wide array of learning tasks. The authors called this phenomenon the Neural Feature Ansatz (NFA).  In this work, we prove that NFA is indeed true in our context: the covariance matrix of weights in layer $\ell$ of a depth $L$ linear neural network $f(A)=\langle A, W_LW_{L-1}\ldots W_1\rangle$ trained by gradient flow and the AGOP of the network taken with respect to the input of layer $\ell$ are equal when using the scaling $\phi(s)=s^{\frac{1}{L - i + 1}}$.  We then introduce deep lin-RFM, a variant of lin-RFM using deep linear predictors that accurately captures the implicit bias of deep linear networks when using AGOP scaling suggested by the deep NFA.

\item {\bf (RFM generalizes IRLS)}
    When $\phi$ is a power with $\phi(s) = (s + \epsilon)^\alpha$ and $\alpha\in (0,\frac{1}{2}]$, we recognize lin-RFM as a re-parameterization of the classical Iteratively Reweighted Least Squares (IRLS-$p$) algorithm from~\cite{MatrixCompletionIterativeLeastSquaresMohan} with $p = 2 - 4\alpha$. Therefore, all convergence guarantees for IRLS-$p$ apply directly.

\item {\bf (SVD-free implementation of lin-RFM)} We present an SVD-free implementation of lin-RFM when $\phi$ is a  power that is a multiple of $\frac{1}{2}$ that scales to matrices with millions of missing observations. Our implementation is faster than IRLS and results in improved empirical performance over deep linear networks for matrix completion tasks.  To the best of our knowledge, when $\phi(s) = s^{\frac{1}{2}}$, our algorithm is the first SVD-free formulation of IRLS-$0$ corresponding to log determinant minimization.
\end{enumerate}

Our results demonstrate that lin-RFM finds sparse/low rank solutions by explicitly utilizing update rules arising from the fixed point equations of explicit regularized objectives.   In contrast, deep learning algorithms do not utilize such explicit updates and instead, rely on implicit biases arising through gradient-based training methods.  As a result, it is not always the case that deep neural networks trained using gradient descent find sparsity inducing solutions.  For example, the flat minima of deep linear diagonal networks correspond to solutions that induce less sparsity than minimizing $\ell_1$ norm~\cite{ding2022flat}.  Similarly, choosing small step size with randomly initialized deep linear diagonal networks can result in solutions that are less sparse than minimizing $\ell_1$ norm~\cite{LinearDiagonalNetsStepSize}. Our results establish an intriguing connection between feature learning in neural networks and sparse recovery. Furthermore, our finding that the specialization of RFM to linear models leads to a generalization of classical sparse learning algorithms provides further evidence for the centrality of AGOP and NFA to feature learning in neural networks.  

\paragraph{Paper outline.}  We begin with Section~\ref{sec: Related work} outlining work related to low-rank matrix recovery.  In Section~\ref{sec: lin-RFM algorithm and fixed point equation}, we introduce lin-RFM, establish that lin-RFM is a generalization of IRLS, and analyze the fixed point equation of lin-RFM, thereby illustrating how lin-RFM finds low-rank solutions.  In Section~\ref{sec: AGOP and ansatz and Deep lin RFM}, we prove the NFA for deep linear networks and introduce deep lin-RFM, a variant of lin-RFM using deep linear predictors that accurately captures the implicit bias of deep linear networks.  Section~\ref{sec: SVD free lin-RFM and experiments} introduces an SVD-free version of lin-RFM when $\phi(s) = s^{\alpha}$ where $\alpha$ is an integer multiple of $\frac{1}{2}$. This version of lin-RFM is faster than IRLS as it is SVD-free, and we show that it outperforms deep linear networks for sparse linear regression and matrix completion tasks.  We conclude  in Section~\ref{sec: conclusion}.

\subsection{Related work on sparse and low-rank recovery}
\label{sec: Related work}
\paragraph{Classical sparse recovery algorithms.}  A large body of literature has developed and analyzed techniques to overcome the curse of dimensionality by identifying sparse structure from data.  Examples of algorithms for variable selection in linear regression include Lasso~\cite{Lasso}, $\ell_1$ and iterated $\ell_1$ minimization~\cite{candes2008enhancing, donoho1989uncertainty}, utilizing non-convex sparsity inducing penalties~\cite{fan2001variable}, and IRLS~\cite{IRLSSparseRecovery, chartrand2008iteratively, merle1974computational, li1993globally}.  Algorithms for low rank matrix completion include minimizing nuclear norm~\cite{CandesRechtMatrixCompletion, MatrixCompletionNuclearNormTao, MinNuclearNormMinRank}, hard singular value thresholding methods~\cite{SVThresholdingMatrixCompletion} and IRLS-$p$~\cite{MatrixCompletionIterativeLeastSquaresWard, MatrixCompletionIterativeLeastSquaresMohan}.  Much of this  work has been motivated by the fact that minimum $\ell_1$ norm solutions (analogously, minimum nuclear norm solutions for matrix completion) provably lead to exact recovery of sparse solutions under certain assumptions on the data (including the Restricted Isometry Property (RIP))~\cite{CandesRechtMatrixCompletion, candes2005decoding, donoho2006compressed}. 

\paragraph{Deep linear networks.} It has been observed that over-parameterized, deep linear neural networks trained using gradient descent can in practice perform sparse recovery, often empirically exhibiting improved sample efficiency over minimizing $\ell_1$ or nuclear norm~\cite{LinearDiagonalNetKernelRichLazy, MatrixFactorizationNIPS2017, DeepMatrixFactorization}.  Indeed, it has been proved that depth $L$ linear neural networks that fit data while minimizing $\ell_2$ norm of network weights implement linear predictors that minimize an $\ell_p$ pseudo norm with $p=\frac{2}{L}$~\cite{shang2020unified, gunasekar2018implicitconv, dai2021representation}.  For example, when $L = 2$, this corresponds to classical $\ell_1$ minimization.  Under special conditions on initialization and learning rate, these solutions can be found by gradient descent~\cite{gunasekar2018implicitconv, MatrixFactorizationNIPS2017}.  In general, it is not known when gradient descent arrives at sparsity inducing solutions.   For example, the work~\cite{LinearDiagonalNetsStepSize} showed that choosing small step size with randomly initialized linear diagonal networks can result in solutions with less sparsity than minimizing $\ell_1$ norm.  Moreover, the recent work~\cite{ding2022flat} proved that the flat minima for fitting linear diagonal networks with more than two layers corresponds to solutions that are inferior to those obtained by $\ell_1$ norm minimization.  

\paragraph{Low rank structure in neural networks.} Recent work has demonstrated that weight matrices of deep nonlinear networks trained using gradient-based methods are often low rank.  This low dimensional structure arises and has been exploited to rapidly fine-tune transformer based language models~\cite{aghajanyan2020intrinsic, hu2022lora}.  A recent line of work~\cite{DLSRepresentationReLU, 22abbe_staircase, NeuralNetsMultiIndexSGD} has theoretically analyzed the emergence of low rank structure in neural networks trained using gradient descent or its variations in the context of multi-index models.  Another line of work has analyzed the implicit bias of nonlinear fully connected networks that fit data while minimizing the $\ell_2$ norm of network weights~\cite{jacot2023bottleneck, jacot2023implicit, parkinson2023linear}.  Under this condition on the weights, the works~\cite{jacot2023bottleneck, jacot2023implicit} identify notions of rank that arise when considering infinitely deep networks.  The work~\cite{parkinson2023linear} defines a notion of rank for fully connected networks with one nonlinearity using the rank of their expected gradient outer product and demonstrates that adding linear layers induces the emergence of low rank structure.  These analyses provide significant insight into the capabilities of neural networks. 
Nevertheless, to the best of our knowledge, there are no guarantees that gradient descent finds solutions that minimize $\ell_2$ norm of network weights. On the other hand, the works~\cite{radhakrishnan2022feature, beaglehole2023mechanism} empirically demonstrate that AGOP accurately captures the emergence of low dimensional structure in weight matrices of fully connected networks and state-of-the-art convolutional networks through training.

\section{Lin-RFM algorithm and fixed point equation}
\label{sec: lin-RFM algorithm and fixed point equation}

\subsection{Problem setup}
\label{sec: problem setup}

In this work, we consider the standard problem of low-rank matrix recovery.  Let $Y_{\sharp} \in \mathbb{R}^{d_1 \times d_2}$ denote a rank $r$ matrix, and suppose without loss of generality $d_1\leq d_2$.  The goal is to recover $Y_{\sharp}$ from a set of $n$ measurements of the form $y_i=\langle A_i, Y_{\sharp} \rangle$, where $\langle A_i, Y_{\sharp} \rangle$ denotes the standard trace inner product $\langle A, Y \rangle = \Tr(A^T Y)$. That is, we would like to recover $Y_{\sharp}$ from the feature/label pairs $\{(A_i, y_i) \}_{i=1}^{n}$ under the auxiliary information that $Y_{\sharp}$ has a low rank.  The matrices $\{A_i\}_{i=1}^{n}$ are called measurement or sensing matrices in the literature.  In our examples and experiments, we will primarily focus on the {\it matrix completion} problem where the measurement matrices are indicator matrices $M_{ij}$ containing a $1$ in coordinate $(i,j)$ and zeros elsewhere.  Note that $\langle M_{ij}, Y_{\sharp} \rangle = {Y_{\sharp}}_{ij}$, and our goal reduces to filling in missing entries in $Y_{\sharp}$ after observing $n$ entries.  

\medskip

\subsection{Lin-RFM algorithm}
\label{sec: lin-rfm algorithm}

\noindent We start by introducing the lin-RFM algorithm for low-rank matrix recovery. In this setting, we focus on linear prediction functions
$$\mathcal{F}:=\{f(A)=\langle AM,W\rangle: W\in \mathbb{R}^{d_1\times d_2},~ M\in \mathbb{R}^{d_2\times d_2}\}.$$
As explained in the introduction, we will think of $f$ as acting on the product space $\mathcal{X}:=\underbrace{\mathbb{R}^{d_2}\times \ldots\times \mathbb{R}^{d_2}}_{d_1~\text{times}}$, with the factors encoding the rows of the matrix. The AGOP of $f$ can then be readily computed from Eq.~\eqref{eq: Generalization of AGOP}. Namely, setting $Z:=WM$ and letting $e_{\ell} \in \mathbb{R}^{d_1}$ denote a standard basis vector, whose coordinates are zero except the $\ell$\textsuperscript{th} coordinate which is $1$, we compute
\begin{align*}
\text{AGOP}(f, (A_1, \ldots A_n)) &= \frac{1}{n d_2} \sum_{i=1}^{n} \sum_{\ell=1}^{d_2} \nabla_{\ell} f(A_i) \nabla_{\ell} f(A_i)^T
=\frac{1}{n d_2} \sum_{i=1}^{n} \sum_{\ell=1}^{d_2} Z^{\top} 
 e_{\ell} e_{\ell}^\top Z=Z^\top Z.
\end{align*}
Therefore a direct application of RFM~\eqref{eqn:RFM}, where the prediction (Step 1) is implemented by finding the smallest Frobenius norm matrix $W$ that interpolates the measurements yields the following algorithm.

\begin{tcolorbox}
\begin{alg}[Lin-RFM]  Let $M_0 = I$ and for $t \in \mathbb{Z}_+$ repeat the steps:
\label{alg: RFM matrix completion}
\begin{align*}
    & \textit{Step 1 (Linear regression):} \qquad W_t = \argmin_{W \in \mathbb{R}^{d_1 \times d_2}}~ \|W\|^2_F  ~~ \text{subject to $\langle A_i M_t , W \rangle  = y_i$ for all $i \in [n]$}~; \\
    & \textit{Step 2 (AGOP Update):} \qquad M_{t+1} = \phi \left( M_t^T W_t^T W_t M_t \right)~.
\end{align*}
\end{alg}
\end{tcolorbox}

When $\phi$ is a power function, $\phi(s) = s^{\alpha}$,  lin-RFM is a re-parameterized version of the prominent IRLS-$p$ algorithm with $p = 2 - 4\alpha$ (see Appendix~\ref{appendix: Reformulation of lin-RFM as IRLS}).  Thus for specific choices of $\phi$, lin-RFM already provides a provably effective approach to low-rank matrix recovery.  In particular, the convergence analyses for IRLS from~\cite{ MatrixCompletionIterativeLeastSquaresMohan} automatically establish convergence for lin-RFM with $\phi(s) = s^{\alpha}$, $\alpha \in \left[ \frac{1}{4}, \frac{1}{2}\right]$, and $M_0 = I$. Our formulation, however, easily extends lin-RFM to low-rank inducing objectives corresponding to negative values of $p$, which have not been considered for linear regression with the exception of~\cite{rao1999affine} and to the best of our knowledge, not at all for matrix completion tasks.    Moreover, in Section~\ref{sec: SVD free lin-RFM and experiments} we will present an efficient, SVD-free implementation of lin-RFM that is particularly effective for matrix completion.

\begin{remark}{\rm The problem of sparse linear regression corresponds to the setting where $Y_{\sharp}$ and $A_i$ are diagonal matrices. In this setting, the update matrices $M_t$ remain diagonal as well. Therefore, Lin-RFM applies directly to problems of sparse linear regression and again 
reduces to IRLS. See Appendix~\ref{appendix: lin-RFM linear reg} for details. }
\end{remark}

\subsection{Fixed point equation of lin-RFM}
\label{sec: fixed point equation of lin-RFM}

We now derive the fixed point equation of lin-RFM, which shows that the fixed points of lin-RFM correspond to the critical points of sparsity inducing objectives depending on the choice of operator $\phi$ applied to AGOP.  In our analysis, we require $\phi\colon\mathbb{R}\to (0,\infty)$ to be a continuous function. By convention,
$\phi$ operates on matrices through eigenvalues, i.e., each eigenvalue, $\lambda_i$, of AGOP is transformed to $\phi(\lambda_i)$ while eigenvectors remain unchanged.  

It will be convenient to rewrite the iterations of 
Algorithm~\ref{alg: RFM matrix completion} as follows. Observe that since $\phi$ takes only positive values, the matrix $M_t$ is nonsingular. Consequently, upon making the variable substitution $Z_t:=W_tM_t$ the iterations of Algorithm~\ref{alg: RFM matrix completion} can be purely written in terms of the evolution of $Z_t$:
\begin{equation}\label{eqn:simplerewrite}
Z_{t+1} = \argmin_{Z \in \mathbb{R}^{d_1 \times d_2}}~ \|Z\phi(Z^\top_t Z_t)^{-1}\|^2_F  \qquad \text{subject to \qquad$\langle A_i , Z \rangle  = y_i$ ~~ $\forall i \in [n]$}.
\end{equation}
Throughout, we will let $\sigma(Z)$ denote the vector of singular vectors of $Z$ in nondecreasing order.

The following theorem characterizes the fixed point equations of algorithm \eqref{eqn:simplerewrite}  as minimizers of a regularized optimization problem \eqref{eqn:regular}. 

\begin{theorem}
\label{theorem: Fixed point equation of lin-RFM}
The fixed points $Z$ of \eqref{eqn:simplerewrite} are first-order critical points of the optimization problem: 
    \begin{equation}\label{eqn:regular}
        \text{$\argmin_{Z \in \mathbb{R}^{d_1 \times d_2}} \sum_{j=1}^{d_1} \psi(\sigma_j(Z)) $\qquad subject to\qquad $\langle Z, A_i \rangle = y_i$ for $i \in [n]$,} 
    \end{equation}    
where we define the function $\psi(r) = \int_{0}^{r} \frac{s}{[\phi(s^2)]^2} ds$ for all $r \in \mathbb{R}$.
\end{theorem}

\begin{proof}
Note that $\psi$ is well-defined and $C^1$-smooth with $\psi'(r)=\frac{r}{[\phi(r^2)]^2}$, since $\phi$ is strictly positive and continuous.
Let $Z$ be a fixed point of the iteration \eqref{eqn:simplerewrite} and let $\mathcal{A}$ denote the linear map encoding the constraints of \eqref{eqn:simplerewrite}, that is $\mathcal{A}(Z)=(\langle A_i,Z\rangle, \ldots,\langle A_n,Z\rangle)$. Then first-order optimality conditions for \eqref{eqn:simplerewrite} imply that there exists a vector of multipliers $\lambda\in\mathbb{R}^n$ satisfying
\begin{equation}\label{eqn:optcond}
    \begin{aligned}
        2Z\phi(Z^{\top}Z)^{-2}+ \mathcal{A}^*(\lambda) &= 0 ~; \\
        \mathcal{A}(Z) &= y~.
    \end{aligned}
    \end{equation}
We claim now that $Z\phi(Z^{\top}Z)^{-2}$ is the gradient of the function $F(Z):=\sum_{j=1}^{d_1} \psi(\sigma_j(Z))$. Indeed the result in \cite[Theorem 1.1]{LewisSpectralDerivatives} tells us that 
$F$ is differentiable, since $\psi$ is differentiable, and moreover we have $\nabla F(Z)=\psi'(Z)=Z\phi(Z^{\top}Z)^{-2}$. Thus, $\eqref{eqn:optcond}$ reduce to the optimality conditions for the problem \eqref{eqn:regular}.  
\end{proof}

Intuitively, the above minimization problem can induce sparsity if $\psi(r)$ is a function that is minimized at $r = 0$ and increases rapidly away from $r = 0$.  A particularly important setting, which we consider next corresponds to power functions $\phi(s) = (s+\epsilon)^{\alpha}$. 

\begin{corollary}
\label{corollary: Fixed point matrix powers}
   Set $\phi(s) := (s + \epsilon)^{\alpha}$ for arbitrary choice of $\epsilon > 0$ and $\alpha>0$. 
   Then, the fixed points $Z$ of \eqref{eqn:simplerewrite} are first-order critical points of the following optimization problems depending on the value of $\alpha$: 
    \begin{align}
        &\text{If $\alpha \neq \frac{1}{2}$:}~~ \argmin_{Z \in \mathbb{R}^{d_1 \times d_2}} ~\text{ $\sum_{j=1}^{d_1} \frac{1}{2 - 4\alpha} $}(\sigma_j^2(Z) + \epsilon)^{1 - 2\alpha} ~~ \text{subject to $\langle Z, A_i \rangle = y_i$ for $i \in [n]$;} \\
        &\text{If $\alpha = \frac{1}{2}$:}~~ \argmin_{Z \in \mathbb{R}^{d_1 \times d_2}} ~\text{$\sum_{j=1}^{d_1}\log(\sigma_j^2(Z) + \epsilon)$} ~~ \text{subject to $\langle Z, A_i \rangle = y_p$ for $i \in [n]$.} \nonumber 
    \end{align}
\end{corollary}
\noindent Note that the factor $\frac{1}{2 - 4\alpha}$ is negative for $\alpha>1/2$.  
\paragraph{Remarks.} The proof of Corollary~\ref{corollary: Fixed point matrix powers} follows immediately from integrating $\psi'(r) = \frac{r}{(r^2 + \epsilon)^{2\alpha}}$.   We note that the $\epsilon$ term is not necessary in practice for convergence in matrix completion tasks even when $\psi(r)$ has a singularity.  Indeed, in Appendix~\ref{sec: lin-RFM matrix completion convergence}, we present examples for matrix completion in which Algorithm~\ref{alg: RFM matrix completion} provably converges with $\epsilon = 0$ and $\psi(r) = -\frac{1}{2r^2}$ when $\alpha = 1$. These examples show that lin-RFM can recover low rank solutions when minimizing nuclear norm does not.  In Appendix~\ref{sec: lin-RFM matrix completion convergence}, we also show that lin-RFM can exhibit behavior similar to deep linear networks used for matrix completion that do not learn solutions that minimize $\ell_2$ norm of weights, including those from~\cite{razin2020implicit}.  For matrix completion problems, we conjecture that lin-RFM with $\epsilon = 0$ is able to recover the ground truth low rank matrix $Y_{\sharp}$ with high probability under standard incoherence assumptions and with sufficiently many observed entries. Note that values of $\alpha \leq \frac{1}{4}$ recover fixed point equations of objectives corresponding to spectral $\ell_p$ norm minimization for $p \in [1, \infty)$.  In particular, for $\alpha = \frac{1}{4}$, we recover the widely used nuclear norm minimization objective.  When $\alpha = \frac{1}{2}$, we recover the fixed point equation of the log determinant objective, which has been analyzed previously as a heuristic for rank minimization~\cite{fazel2003logdet}.  This objective also arises by considering the limit as $p \to 0$ in the IRLS-$p$ algorithm~\cite{MatrixCompletionIterativeLeastSquaresMohan}.  

 Note that lin-RFM is a generalization of IRLS.  In particular, lin-RFM with $\phi(s) = (s + \epsilon)^{\alpha}$ is a re-parameterization of IRLS-$p$ \cite{MatrixCompletionIterativeLeastSquaresMohan} with $p = 2-4\alpha$.  For $\alpha > \frac{1}{2}$, lin-RFM corresponds to IRLS-$p$ with negative values of $p$, which has not been used in the literature to the best of our knowledge.  Observe that lin-RFM and IRLS find sparsity inducing solutions by explicitly utilizing update rules based on the fixed point equation of sparsity inducing objectives.  In contrast, deep linear networks are implicitly biased towards such sparsity inducing solutions \textit{if} gradient descent finds solutions which minimize $\ell_2$ norm of the weights.  Under this weight minimization condition, depth $L$ linear neural networks for $L \in \{2, 3, \ldots \}$ used for matrix completion learn matrices, $\hat{W}$, with minimum $\ell_{2/L}$ (pseudo-)norm on singular values~\cite{shang2020unified, gunasekar2018implicitconv}.

\section{Deep Neural Feature Ansatz for deep linear networks and deep lin-RFM}
\label{sec: AGOP and ansatz and Deep lin RFM}

Deep Neural Feature Ansatz (NFA) was introduced in~\cite{radhakrishnan2022feature} as a guiding principle that led to the development of RFM. It was shown to hold approximately in various architectures of neural networks~\cite{radhakrishnan2022feature, beaglehole2023mechanism}.  In this section, we prove that the NFA holds exactly for the special case of deep linear networks used for low-rank matrix recovery.  Our result connects deep linear networks, lin-RFM, and consequently IRLS, by demonstrating that all models utilize AGOP in their updates.  The NFA connected covariance matrices of weights in trained neural networks with the AGOP as follows.  
\begin{tcolorbox}
\textbf{Deep Neural Feature Ansatz (NFA).} Given an elementwise nonlinearity $\sigma: \mathbb{R} \to \mathbb{R}$, let $f(x)$ denote a depth $L$ neural network of the form 
\begin{align}
\label{eq: deep neural net}
    f(x) = h_{L}(x) ~ ; ~ h_{\ell}(x) = W^{(\ell)} \sigma( h_{\ell-1}(x) ) ~ \text{for $\ell \in \{2, 3, \ldots, L\}$};
\end{align}
where $h_1(x) = W^{(1)}x$.  Suppose that $f$ is trained using gradient descent on data $(x_1, y_1), \ldots, (x_n, y_n) \in \mathbb{R}^{d} \times \mathbb{R}$.  Let $f_t$ denote the predictor and let $W_t^{(\ell)}$ for $\ell \in [L]$ denote the weights after $t$ steps of gradient descent.  Then,
\begin{align}
\label{eq: NFA}
    {W_t^{(\ell)}}^T {W_t^{(\ell)}} \propto \left(\text{AGOP}\left(f_t, (g_{\ell}(x_1) , \ldots, g_{\ell}(x_n))\right)\right)^{\alpha_{\ell}}~;\tag{{NFA}}
\end{align}
for matrix powers $\alpha_{\ell} > 0$ where $g_{\ell}(x) := \sigma(h_{\ell-1}(x))$ for $\ell \in \{2, 3, \ldots, L\}$ and $g_1(x) := x$.  
\end{tcolorbox}
In the case of deep linear networks used for matrix sensing, Eq.~\eqref{eq: deep neural net} reduces to $f(A) = \langle A, W_L \ldots W_1 \rangle$.  Noting that $W_{\ell}^T$ operates on the rows of $AW_1^T \ldots W_{\ell-1}^T W_{\ell}^T$ and repeating the argument from Section~\ref{sec: lin-rfm algorithm}, the AGOP is given by ${W^{(\ell)}}^{T} \ldots {W^{(L)}}^T W^{(L)} \ldots W^{(\ell)}$.  Hence, NFA amounts to the assertion
\begin{align}
\label{eq: NFA deep linear network}
    {W_t^{(\ell)}}^T {W_t^{(\ell)}} \propto \left({W_t^{(\ell)}}^{T} \ldots {W_t^{(L)}}^T W_t^{(L)} \ldots W_t^{(\ell)} \right)^{\alpha_{\ell}}~\qquad \forall \ell\in [L].
\end{align}
The following theorem shows that the NFA indeed holds with the specific powers $\alpha_{\ell} = \frac{1}{L - \ell + 1}$. 
\begin{theorem}
\label{thm: NFA matrix completion}    
    Let $f(A) =  \langle A, W^{(L)} W^{(L-1)} \ldots W^{(1)}\rangle$ denote an $L$-layer linear network.  Let $f_t$ denote the network trained for time $t$ by continuous-time gradient flow on the mean square loss. If the initial weight matrices $\{W_0^{(\ell)}\}_{\ell=1}^{L-1}$ are balanced in the sense that ${W_0^{(\ell)}} {W_0^{(\ell)}}^T = {W_0^{(\ell+1)}}^T {W_0^{(\ell+1)}}$, then for any $t > 0$,
    \begin{align}
        \label{eq: ansatz deep linear nets for matrix completion}
            {W_t^{(\ell)}}^T W_t^{(\ell)} = \left({W_t^{(\ell)}}^{T} \ldots {W_t^{(L)}}^T W_t^{(L)} \ldots W_t^{(\ell)} \right)^\frac{1}{L - \ell + 1}~.  
    \end{align}
\end{theorem}
\begin{proof}    
The key to the proof is that  when training with gradient flow starting from balanced weights, the weights remain balanced for all time: ${W_t^{(\ell)}} {W_t^{(\ell)}}^T = {W_t^{(\ell+1)}}^T {W_t^{(\ell+1)}}$ for any $t > 0$~\cite[Appendix A.1]{Balancedness}.  For the reader's convenience, we restate and simplify the argument for this result in Appendix~\ref{sec: proof of balancedness}.   The proof of Theorem~\ref{thm: NFA matrix completion} follows immediately from balancedness of weights as the right hand side of Eq.~\eqref{eq: ansatz deep linear nets for matrix completion} simplifies to 
    \begin{align*}
        \left( {W_t^{(\ell)}}^{T} \ldots {W_t^{(L)}}^T W_t^{(L)} \ldots W_t^{(\ell)} \right)^\frac{1}{L - \ell + 1} &= \left( {W_t^{(\ell)}}^{T} \ldots {W_t^{(L-1)}}^T \left( W_t^{(L-1)} {W_t^{(L-1)}}^T \right) W_t^{(L-1)} \ldots W_t^{(\ell)} \right)^\frac{1}{L - \ell + 1} \\
        &= \left( {W_t^{(\ell)}}^{T} \ldots \left({W_t^{(L-1)}}^T W_t^{(L-1)} \right)^2 \ldots W_t^{(\ell)} \right)^\frac{1}{L - \ell + 1}\\
        &= \left( {W_t^{(\ell)}}^{T} \ldots {W_t^{(L-2)}}^T \left({W_t^{(L-2)}} {W_t^{(L-2)}}^T \right)^2 W_t^{(L-2)} \ldots W_t^{(\ell)} \right)^\frac{1}{L - \ell + 1} \\
        &= \left( {W_t^{(\ell)}}^{T} \ldots \left({W_t^{(L-2)}}^T {W_t^{(L-2)}}\right)^3  \ldots W_t^{(\ell)} \right)^\frac{1}{L - \ell + 1}\\
        &~ \vdots \\
        &= \left({W_t^{(\ell)}}^T W_t^{(\ell)} \right)^{\frac{L- \ell + 1}{L - \ell + 1}} \\
        &= {W_t^{(\ell)}}^T W_t^{(\ell)}~,
    \end{align*}
    which completes the argument. 
\end{proof}
\textbf{Remarks.}  The work~\cite{ApproximateBalancedness} proved that an approximate form of balancedness, known as $\delta$-balancedness, holds when training linear networks using discrete time gradient descent.  In particular, $\delta$-balancedness of weights is defined in~\cite{ApproximateBalancedness} as $\left\|{W^{(\ell+1)}}^T W^{(\ell+1)} - W^{(\ell)} {W^{(\ell)}}^T\right\|_F \leq \delta$.  We note that Theorem~\ref{thm: NFA matrix completion} also holds under $\delta$-balancedness upon including an $O(\delta)$ correction to the gradient outer product.  The above result simplifies to the case of deep linear diagonal networks used for linear regression by letting $f(x) = \langle x, W^{(L)} W^{(L-1)} \ldots W^{(1)} \rangle$ where the weights $\{W^{(i)}\}_{i=1}^{L}$ and data $x$ are diagonal matrices.  Note that in the case of deep linear diagonal networks, the balancedness condition on initialization implies that the magnitude of weights are equal at initialization.

The fact that the NFA holds in deep linear networks suggests that we should be able to streamline the training process for deep linear networks by utilizing AGOP instead of gradient descent.  To this end, we next introduce a formulation of lin-RFM for predictors of the form $f(A) = \langle A, W M^{(L)} \ldots M^{(1)}\rangle$, which we refer to as \textit{deep lin-RFM}. 

\begin{tcolorbox}
\begin{alg}[Depth $L$ lin-RFM] Let $f(A) = \langle A, W M^{(L)} \ldots M^{(1)}\rangle$, $\{\phi_{\ell}\}_{\ell=1}^{L}$ denote univariate functions operating on matrices through eigenvalues, and $\{(A_i, y_i)\}_{i=1}^{n} \subset \mathbb{R}^{d_1 \times d_2} \times \mathbb{R}$ denote training data.  Let $M_0^{(\ell)} = I$ for $\ell \in [L]$ and for $t \in \mathbb{Z}_+$ repeat the steps:
\label{alg: deep lin RFM matrix completion}
\begin{align*}
    & \textit{Step 1 (Linear Regression):} ~~ W_t = \argmin_{W \in \mathbb{R}^{d_1 \times d_2}}~ \|W\|^2_F  ~~ \text{subject to $\langle A_i, W M_t^{(L)} \ldots M_t^{(1)}\rangle = y_i$, $i \in [n]$}; \\
    & \textit{Step 2 (Layer-wise AGOP Update):} ~~ M_{t+1}^{(\ell)} = \phi_{\ell} \left( {M_t^{(\ell)}}^T \ldots {M_t^{(L)}}^T W_t^T W_t M_t^{(L)} \ldots M_t^{(\ell)} \right), \text{$\ell \in [L]$}.
\end{align*}
\end{alg}
\end{tcolorbox}
Analogously to the case of lin-RFM in Section~\ref{sec: fixed point equation of lin-RFM}, we can derive the fixed point equation of deep lin-RFM when $\phi_{\ell}(s) = (s + \epsilon)^{\alpha_{\ell}}$.
\begin{theorem}
\label{thm: deep lin-RFM fixed point analysis}
    Let $\phi_{\ell}(s) = (s + \epsilon)^{\alpha_{\ell}}$ for $\epsilon > 0$ and let $C_{\ell} = \sum_{i=1}^{\ell} \prod_{j=1}^{i-1} \alpha_i (1 - 2\alpha_j)$.  Assume $C_{\ell} > 0$ for all $\ell \in [L]$ and that $\sum_{\ell=1}^{L} C_{\ell} < \frac{1}{2}$.  Then, the fixed points of Algorithm~\ref{alg: deep lin RFM matrix completion} are the first-order critical points of the following optimization problem: 
    \begin{align*}
        \argmin_{Z \in \mathbb{R}^{d_1 \times d_2}} \sum_{j=1}^{d_1} \psi_{\epsilon}(\sigma_j(Z)) ~\text{subject to $\mathcal{A}(Z) = y$}~;
    \end{align*}    
    where $\psi_{\epsilon}(r) := \int_{0}^{r} s h_1(s) \ldots h_L(s)\, ds$ and $h_i(s)$ are defined recursively as 
    \begin{align}
        h_1(s) = (s^2 + \epsilon)^{\alpha_1};\qquad  h_{\ell}(s) = \left( s^2 h_1(s)^{-2} \ldots h_{\ell-1}(s)^{-2} + \epsilon \right)^{\alpha_{\ell}} ~ \text{for $\ell \in \{2, \ldots, L\}$}.
    \end{align}
  Moreover, $\lim\limits_{\epsilon \to 0} \psi_{\epsilon}(r) = r^{2 - 4C_1 - \ldots -4C_L}$ for every $r$.  
\end{theorem}
Here, the condition on the terms $C_{1}, \ldots, C_{L}$ is such that we may establish convergence of $\psi_{\epsilon}$ as $\epsilon \to 0$.  The following corollary to Theorem~\ref{thm: deep lin-RFM fixed point analysis} shows that using powers $\alpha_{\ell}$ suggested by Theorem~\ref{thm: NFA matrix completion} captures the implicit bias of depth $L+1$ linear networks that fit data while minimizing $\ell_2$ norm of network weights.   
\begin{corollary}    
    \label{corollary: deep lin rfm and deep lin networks}
    Let $\phi_{\ell}(s) = (s + \epsilon)^{\alpha_{\ell}}$ for $\epsilon > 0$ where $\alpha_{\ell} = \frac{1}{2(L - \ell + 2)}$.  Then, the fixed points of Algorithm~\ref{alg: deep lin RFM matrix completion} are the first-order critical points of the following optimization problem: 
    \begin{align*}
        \argmin_{Z \in \mathbb{R}^{d_1 \times d_2}} \sum_{j=1}^{d_1} \psi_{\epsilon}(\sigma_j(Z)) ~\text{subject to $\mathcal{A}(Z) = y$}~;
    \end{align*}    
    where $\psi_{\epsilon}$ is defined in Theorem~\ref{thm: deep lin-RFM fixed point analysis} and satisfies $\lim\limits_{\epsilon \to 0} \psi_{\epsilon}(r) = r^{\frac{2}{L+1}}$ for every $r$.    
\end{corollary}
The proof of Corollary~\ref{corollary: deep lin rfm and deep lin networks} follows from verifying that $\alpha_{\ell} = \frac{1}{2(L - \ell +2)}$ implies $C_{\ell} = \frac{1}{2(L+1)}$ for $\ell \in [L]$.  This result provably shows that the implicit bias of deep linear networks can be recovered without backpropagation, by solely training the last layer of deep networks and utilizing AGOP for updating intermediate layers. Note that deep lin-RFM with arbitrary functions $\phi_{\ell}$ is more general than shallow lin-RFM analyzed in Algorithm~\ref{alg: RFM matrix completion}.  Nevertheless, for matrix powers, for $\phi_{\ell}(s) = (s + \epsilon)^{\alpha_{\ell}}$, the fixed point equation of deep lin-RFM can be expressed in terms of shallow lin-RFM for a specific choice of $\alpha$, namely, $\alpha = \sum_{\ell=1}^{L}C_{\ell}$, as is evident from the form of $\psi_{\epsilon}(r)$ as $\epsilon \to 0$ in Theorem~\ref{thm: deep lin-RFM fixed point analysis}.

\section{SVD-free lin-RFM and numerical results}
\label{sec: SVD free lin-RFM and experiments}

\begin{figure}[!t]
    \centering
    \includegraphics[width=\textwidth]{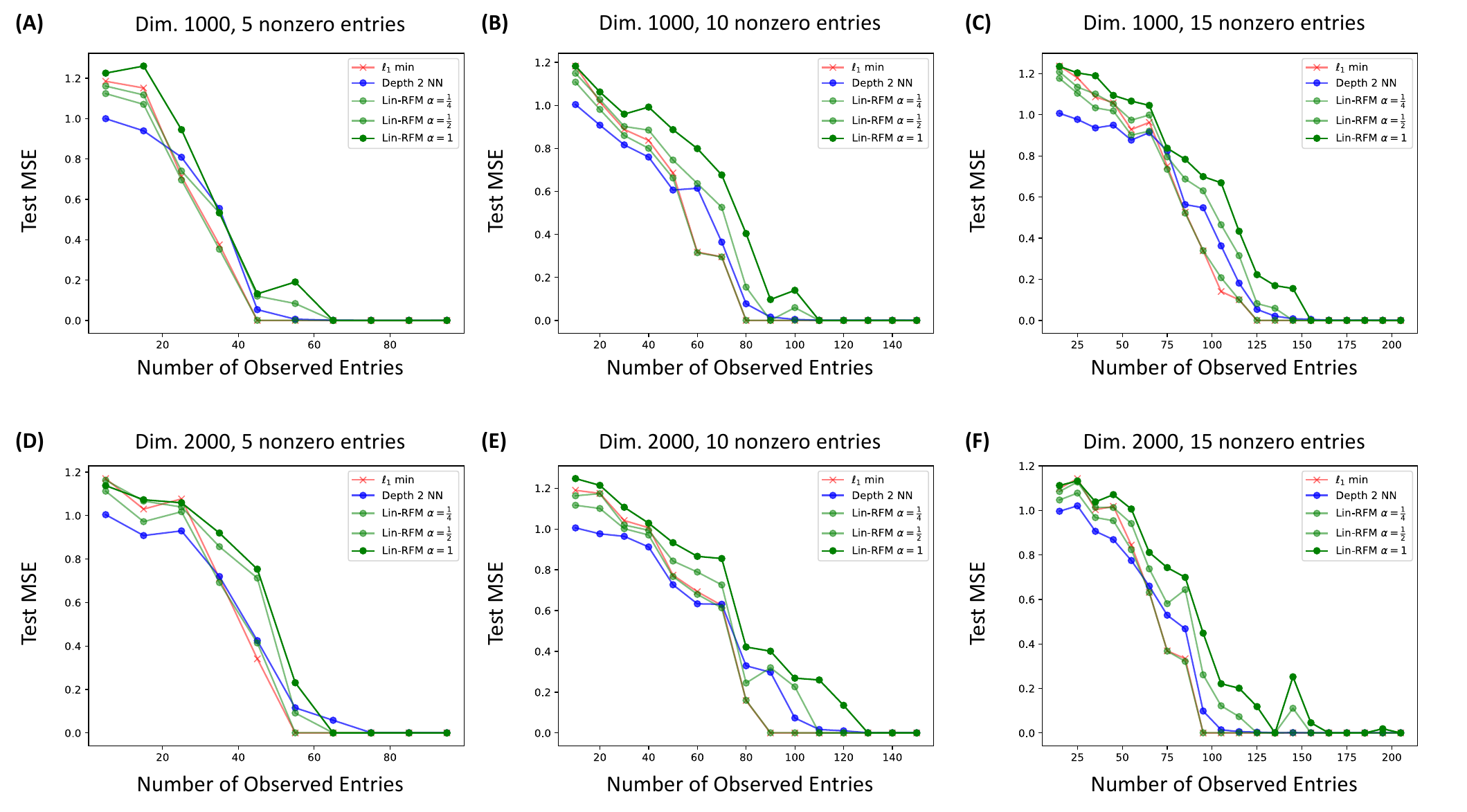}
    \caption{Performance of lin-RFM with various matrix powers $\alpha$, deep linear networks, and minimizing $\ell_1$ norm in sparse linear regression.  All models are trained and tested on the same data, and all results are averaged over 5 random draws of data.}
    \label{fig: NoiseLess linear reg}
\end{figure}

We now introduce a computationally efficient, SVD-free formulation of Algorithm~\ref{alg: RFM matrix completion} when selecting $\phi$ to be matrix powers $\phi(s)=s^{\alpha}$ where  $\alpha$ is any integer multiply of $\frac{1}{2}$.  Notably, our result for $\alpha = \frac{1}{2}$ shows that IRLS-$0$ can be implemented without computing any SVDs.  We then utilize our implementation to empirically analyze how various hyper-parameters of lin-RFM affect sparse recovery.  We empirically show that for noiseless linear regression, $\ell_1$-norm minimization corresponding to lin-RFM with $\alpha = \frac{1}{4}$ yields best results.  We then compare the performance of lin-RFM with $\alpha \in \{\frac{1}{2}, 1\}$, deep linear networks, and nuclear norm minimization for low rank matrix completion.  Our results demonstrate that lin-RFM for $\alpha \geq \frac{1}{2}$ substantially outperforms deep linear networks for this task.

\paragraph{SVD-free form of lin-RFM.}  It is evident from Algorithm~\ref{alg: RFM matrix completion} that lin-RFM with $\phi$ as non-negative integer matrix powers $\alpha$ can be implemented through matrix multiplications.  Thus, the main computational difficulty arises when utilizing non-integral powers $\alpha$, because naively each iteration would require forming a full SVD. We show now how to avoid costly SVD computations entirely.  For simplicity, we focus on lin-RFM for $\alpha = \frac{1}{2}$; similar argument holds for any $\alpha$ that is an integer multiple of $\frac{1}{2}$. We note that in our experiments on matrix completion lin-RFM with $\alpha = \frac{1}{2}$ performs the best in terms of both the recovery error and computational speed.

We begin by observing that the optimality conditions for Step 1 of Algorithm~\ref{alg: RFM matrix completion} read as follows: there exists a vector $\gamma=(\gamma^{(1)},\ldots, \gamma^{(n)})$ satisfying 
\begin{align*}
    W_t=\sum_{j=1}^n \gamma^{(j)} A_jM_t\qquad \textrm{and}\qquad
    y_i=\langle A_iM_t,W_t\rangle\qquad \forall i\in [n].
\end{align*}
Plugging in the first equation into the second, we see that instead of searching for $W_t$ we may search for $\gamma$ satisfying the linear equation
\begin{equation}\label{eqn:reform}
y_i=\sum_{j=1}^n \gamma^{(j)}{\rm tr}(A_j^{\top}M_t^2)\qquad \forall i\in [n].
\end{equation}
Assuming that we have $M_t^2$ available, computing $\gamma$ is just an equation solve. Now Step 2 of Algorithm~\ref{alg: RFM matrix completion} yields the following update equation  
\begin{align*}
M_{t+1}^2=M_t W_t^\top W_t M_t=M_t^2 \left(\sum_{j=1}^n\gamma^{(j)} A_j \right)^\top\left(\sum_{j=1}^n\gamma^{(j)} A_j\right) M_t^2
\end{align*}
Therefore we may compute $M_{t+1}^2$ from $\gamma$ and $M_t^2$ through matrix multiplications, thereby avoiding SVD computations throughout the algorithm.

In our experiments on the matrix completion problem, we add ridge regularization during the equation solve in Eq.~\eqref{eqn:reform}, utilizing a ridge parameter $\lambda > 0$ that is the same for all rows.  We also explicitly enforce that $(W_t M_t)_{\Omega} = Y_{\Omega}$ in each iteration.

\begin{figure}[!t]
    \centering
    \includegraphics[width=\textwidth]{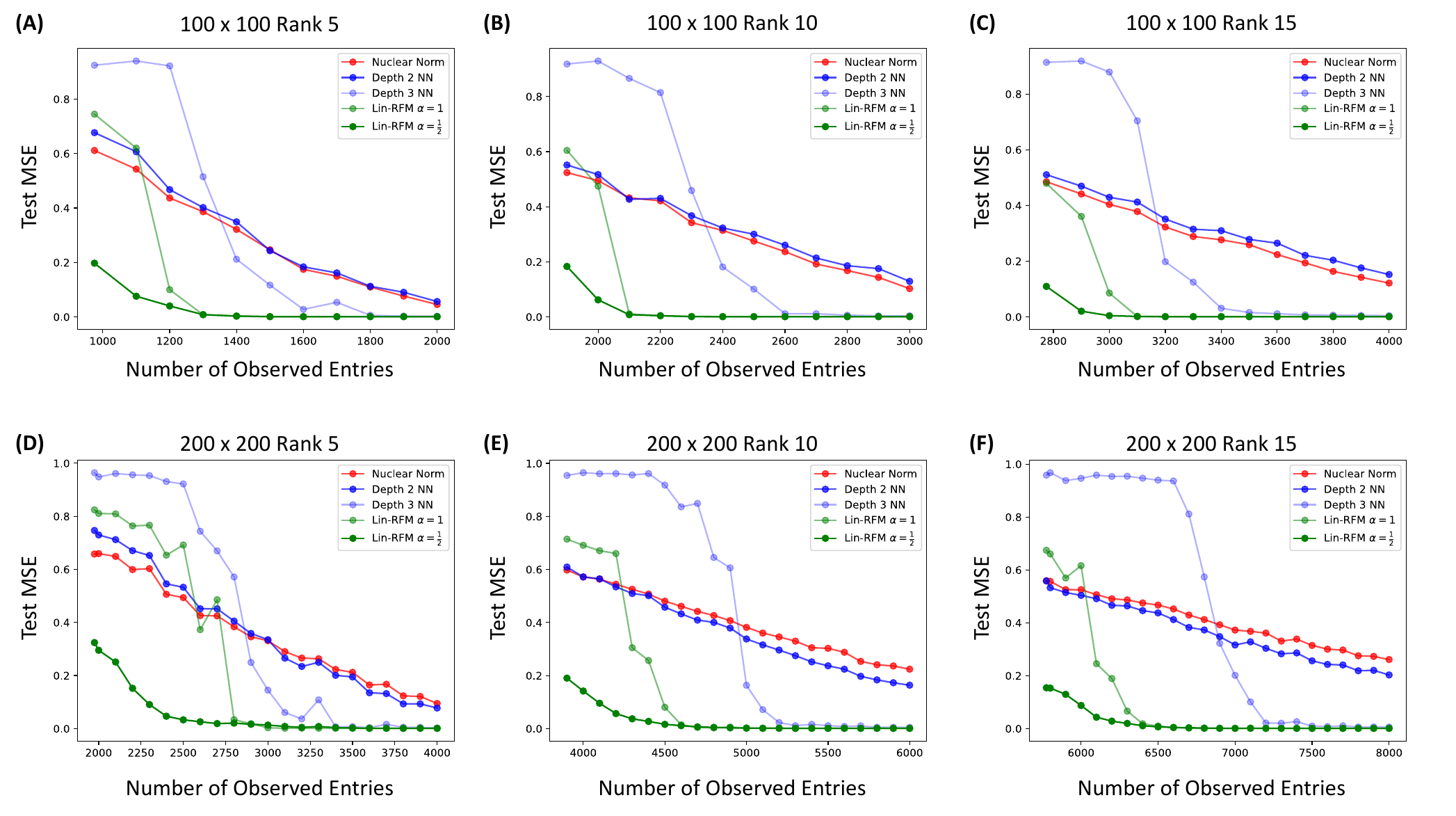}
    \caption{Performance of lin-RFM, deep linear networks, and minimizing nuclear norm in low rank matrix completion.  All models are trained and tested on the same data, and all results are averaged over 5 random draws of data.}
    \label{fig: NoiseLess matrix completion}
\end{figure}

\paragraph{Comparison of models for noiseless linear regression.}  We now compare performance of lin-RFM for $\alpha \in \{ \frac{1}{4}, \frac{1}{2}, 1 \}$, depth $2$ linear diagonal networks (predictors of the form $f(x) = x^T (w_2 \odot w_1)$ where $\odot$ denotes the Hadamard product) trained using gradient descent, and minimizing $\ell_1$-norm in the context of noiseless, sparse linear regression.  In particular, we sample $n$ training samples $X \in \mathbb{R}^{n \times d}$ according to $X_{i, :} \sim \mathcal{N}(0, I_d)$, generate a ground truth sparse weight vector $w^* \in \mathbb{R}^d$ with $w_i^* \sim U[.5, 1]$ for $i \in [r]$ and $w_i^* = 0$ for $i > r$. We choose our sampling scheme for $w^*$ to separate non-zero entries of $w^*$ from $0$.  Our training labels are generated according to $y = X w^*$.  In Fig.~\ref{fig: NoiseLess linear reg}, we compare the test mean squared error (MSE) of these models on a held out test set of size $10000$ for $d \in \{1000, 2000\}$ and $r \in \{5, 10, 15\}$, as a function of the number of training samples $n$.  For lin-RFM, we utilize a regularization parameter of $10^{-10}$ to avoid numerical instabilities in linear regression solves.  Training details for all models are outlined in Appendix~\ref{appendix: Training and hyper-parameter details}.  We compare against two layer diagonal networks for computational reasons: three layer linear diagonal networks took over $10^{6}$ epochs to converge with gradient descent from near-zero initialization, and without such near-zero initialization, performance was far worse (closer to that of $\ell_2$-norm minimization) for three-layer diagonal networks than two layer diagonal networks.  These results are consistent with those from~\cite{LinearDiagonalNetsStepSize}.  As predicted by our results, we find that $\ell_1$-norm minimization and lin-RFM with $\alpha = \frac{1}{4}$ perform similarly (the curves are overlapping).  Interestingly, we observe that models that minimize $\ell_1$-norm achieve best recovery rates for this problem, as defined by average test MSE under $10^{-3}$ across $5$ draws of data.

\begin{figure}[!t]
    \centering
    \includegraphics[width=\textwidth]{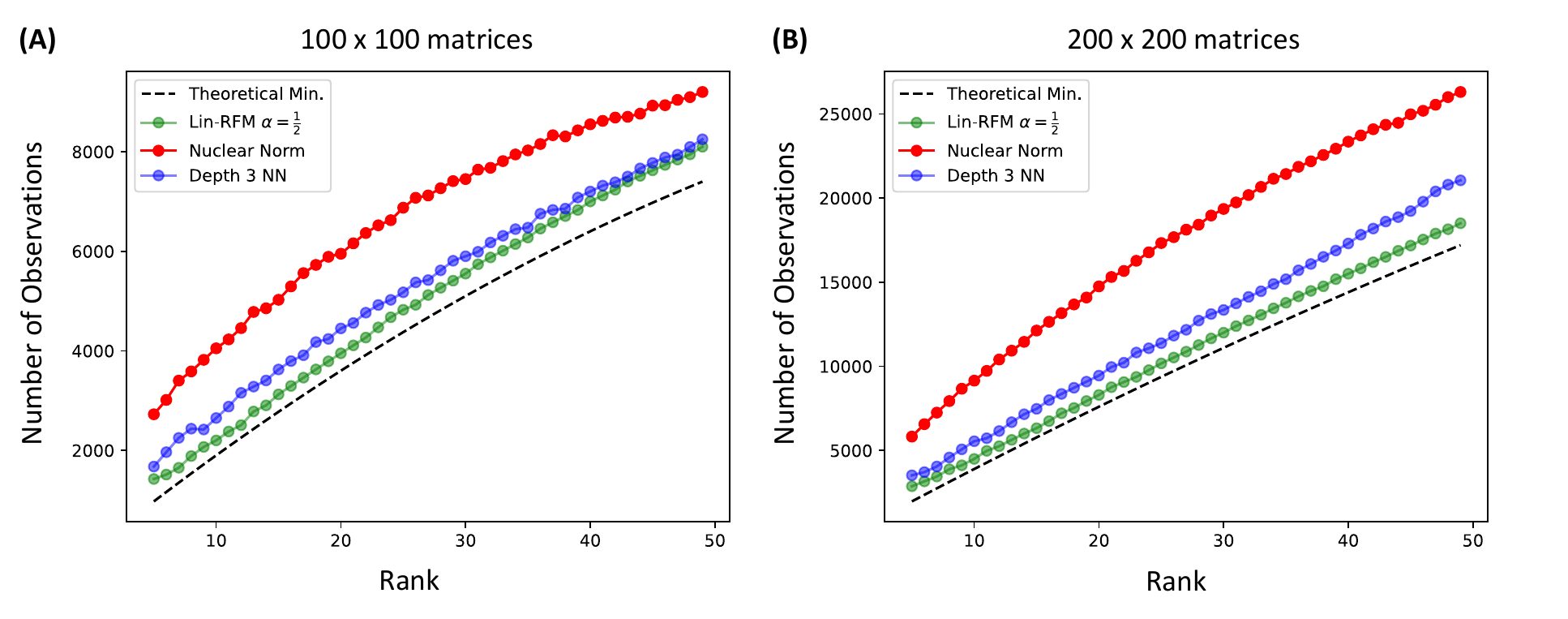}
    \caption{Performance of lin-RFM , deep linear networks, and minimizing nuclear norm in matrix completion as a function of the rank of the ground truth matrix.  The dashed black line represents the number of degrees of freedom, $2dr - r^2$, of a rank $r$ matrix of size $d \times d$.  Note that all curves must converge as $r \to d$ since there is only one rank $d$ solution requiring $d^2$ observations.  Each point represents the number of observations at which the model was able to achieve under $10^{-3}$ test MSE across 5 random draws of data.  Overall, we observe that lin-RFM with $\alpha = \frac{1}{2}$ requires the fewest samples to achieve consistently low test MSE.  Lin-RFM requires up to thousands fewer examples than deep linear networks and directly minimizing nuclear norm.}
    \label{fig: Matrix completion scaling}
\end{figure}

\paragraph{Comparison of models for low rank matrix completion.} We next compare performance of lin-RFM for $\alpha \in \{\frac{1}{2}, 1\}$, depth $2$ and $3$ linear networks trained using gradient descent, and minimizing nuclear norm in the context of low rank matrix completion.  We omit lin-RFM with $\alpha = \frac{1}{4}$ since performance is similar to that of nuclear norm minimization, and we do not have an SVD-free formulation of lin-RFM for this choice of $\alpha$.  In particular, we sample a ground truth $d \times d$ matrix, $Y$, of rank $r$, observe a random subset of $n$ entries, and hold out the remaining $d^2 - n$ entries as test data.  We generate data according to the same procedure in~\cite{DeepMatrixFactorization}.  Namely, we first generate two matrices, $U \in \mathbb{R}^{d \times r}$ and $V \in \mathbb{R}^{d \times r}$, with $U_{ij}, V_{ij} \overset{i.i.d.}{\sim} \mathcal{N}(0, 1)$ for $i \in [d], j \in [r]$, and then, set $Y = \frac{d}{\|UV^T\|_F} U V^T$.  In Fig.~\ref{fig: NoiseLess matrix completion}, we compare the test MSE of these models for $d \in \{100, 200\}$ and $r \in \{5, 10, 15\}$ as a function of the number of observed entries $n$.  In all experiments, we start with $n = 2dr - r^2$, which corresponds to the number of degrees of freedom of a rank $r$ matrix~\cite{MatrixCompletionNuclearNormTao}.  Training and hyper-parameter details for all models are presented in Appendix~\ref{appendix: Training and hyper-parameter details}.  We utilize depth $2, 3$ linear networks in our comparisons in Fig.~\ref{fig: NoiseLess matrix completion} since prior work~\cite{DeepMatrixFactorization} demonstrated that higher depth networks yielded comparable performance to depth $3$ networks.  In contrast to the case of sparse linear regression, we observe that lin-RFM with $\alpha = \frac{1}{2}$ (corresponding to the log-determinant objective) far outperforms other models on this task.  In Appendix Fig.~\ref{fig: Regularization}, we analyze the effect of tuning ridge-regularization parameter in lin-RFM.  We generally observe that as rank increases, lowering regularization leads to better performance.

Next, we compare performance of these models as a function of rank of the ground truth matrix.  In particular, in Fig.~\ref{fig: Matrix completion scaling}, we plot the number of examples needed by lin-RFM, nuclear norm minimization, and deep linear networks to achieve under $10^{-3}$ test error across five random draws of observations.  As a baseline for comparison, we plot (dashed black line) the number of degrees of freedom of a rank $r$ matrix of size $d \times d$, which is $2dr - r^2$.  We observe that lin-RFM requires up to thousands fewer entries than depth $3$ linear networks and directly minimizing nuclear norm for consistently achieve low test error.  In Appendix Fig.~\ref{fig: Label Noise}, we present results for the case of noisy matrix completion in which the training observations are corrupted with additive label noise.  We again find that lin-RFM outperforms all other methods in this setting.

\paragraph{Runtime comparison.}  While Algorithm~\ref{alg: RFM matrix completion} is written concisely, finding the minimum norm solution naively as stated in \eqref{eqn:reform} is inefficient.  Namely, this step involves solving a system of equations with $n = |\Omega|$ constraints and $d_1 d_2$ variables for $d_1 \times d_2$ matrices. Thus, as $n \leq d_1 d_2$, the runtime of naively solving this problem is $O(n^2 d_1 d_2)$.  Instead, we can take advantage of the sparsity in this problem to reduce the runtime to $O(d_2 \sum_{i=1}^{d_1} n_i^2)$ where $n_i$ is the number of observations in row $i$.  In particular, we observe that we can solve for the rows of $W_t$ independently with row $i$ involving a problem with $d_2$ variables and $n_i$ constraints.  Thus, the runtime per row is $O(n_i^2 d_2)$ and for all rows, the runtime is $O(d_2 \sum_{i=1}^{d_1} n_i^2)$.  If $d_1, d_2$ are equal to $d$, and the entries are observed at random concentrating around $\frac{n}{d}$ entries per row, then this runtime is $O(n^2)$, which reduces the runtime by a factor of $d_1 d_2$.  Moreover, the amount of memory needed by solving the rows independently is only $O(d_1 d_2)$, which is in contrast to the $O(n d_1 d_2)$ memory needed by solving the problem naively.  As Step 2 involves matrix multiplication, which is naively $O(d_2^3)$ when multiplying $d_2 \times d_2$ matrices, the total runtime is $O(d_2 \sum_{i=1}^{d_1} n_i^2) + O(d_2^3)$.  As a concrete example of runtimes, for a $5000 \times 5000$ rank $10$ matrix with $1$ million observations, our method takes $559$ seconds, two-layer width $5000$ networks take $5095$ seconds, and three-layer width $5000$ networks take $1779$ seconds to achieve less than $10^{-3}$ test error.  All runtimes are in wall time and all models are trained using a single $12$GB Titan Xp GPU. 

\section{Summary and Broader perspectives}
\label{sec: conclusion}

\paragraph{Summary.}{A key question of modern deep learning is understanding how neural networks automatically perform {\it feature learning} -- task-specific dimensionality reduction through training.  The recent work~\cite{radhakrishnan2022feature} posited that feature learning arises through a statistical object known as average gradient outer product (AGOP).  Based on that insight, the authors proposed Recursive Feature Machines (RFMs), as an iterative algorithm that explicitly performs feature learning by alternating between reweighting feature vectors with AGOP and training predictors on the transformed features. 
 In this paper, we show that RFM for the special case of linear models recovers classical Iteratively Reweighted Least Squares algorithms for sparse linear regression and low-rank matrix recovery. We believe our results are significant for the following reasons: 
\begin{enumerate}
\item The fact that classical sparse learning algorithms naturally arise from a seemingly unrelated feature learning mechanism in deep learning is unlikely to be a fortuitous coincidence. This correspondence provides substantial evidence that the mechanism proposed in~\cite{radhakrishnan2022feature} is indeed a major aspect of feature learning in deep neural networks. 
\item This connection opens a new line of investigation relating still poorly understood aspects of deep learning to classical techniques of statistical inference. Through this connection, RFMs and, by extension, fully connected neural networks, can be viewed as nonlinear extensions of IRLS.  We envision that this connection could be fruitful in understanding the effectiveness of nonlinear deep networks used today and building the next generation of theoretically well-grounded models.  
\item The SVD-free implementation of Lin-RFM for $\alpha=1/2$ provides an efficient approach to solve low-rank matrix recovery problems, and is new to the best of our knowledge. Notably, lin-RFM subsumes IRLS-$p$ as well as implicit bias of deep linear networks identified in~\cite{gunasekar2018implicitconv, shang2020unified}.  
\end{enumerate}}

\paragraph{Broader perspectives on beating the curse of dimensionality.}  Many natural or artificial datasets are high-dimensional. A gray-scale image is a point in a space whose  dimension is equal to the number of pixels, easily in the millions. Representations of text used in modern LLMs and even in older techniques, such as bag-of-words, have anywhere from thousands to billions of dimensions. Dealing with high-dimensional data has long been one of the most significant and difficult problems in statistics and data analysis. Traditional linear methods of dimensionality reduction such as principal components analysis (PCA) are among the most widely used algorithms in data science and, indeed, in all of science. 
They, and their variations, have been useful in machine learning problems such as vision (Eigenfaces~\cite{turk1991eigenfaces}) or Natural Language Processing (NLP) (Latent Semantic Indexing~\cite{deerwester1990indexing}). 
Still, faced with modern data and applications, these methods are generally insufficient as they suffer from two limitations: (a) they are linear with respect to a fixed, pre-determined basis; (b) they are not task-adaptive. 

Historically, two  major lines of research to address these limitations have been non-linear-dimensionality reduction/manifold learning~\cite{wiki:nonlineardimensionalityreduction} and sparse inference~\cite{wiki:compressedsensing}, which is discussed at length in this work. 
It is interesting to note that while manifold learning deals effectively with the non-linearity of the data, it is still primarily an unsupervised technique suffering from the limitation (b). In contrast, the work on sparsity attempts to recover structure from the labeled data and thus addresses the limitation (b). However, much of the sparse inference work (aside from dictionary learning) concentrates on sparsity with respect to a certain basis (e.g., a wavelet basis) selected {\it a priori}, without reference to the data.  It is thus not able to address (a). It appears that modern deep learning is capable of successfully overcoming both (a) and (b).  We hope that the mechanisms exposed in this paper will shed light on the remarkable success of deep learning and point toward designing new better methods for high-dimensional inference.

\section*{Lin-RFM Code}
Code for lin-RFM is available at \url{https://github.com/aradha/lin-RFM}.  

\section*{Acknowledgements} 
A.R. is supported by the George F. Carrier Postdoctoral Fellowship in the School of Engineering and Applied Sciences at Harvard University. 
M.B. acknowledges support from National Science Foundation (NSF) and the Simons Foundation for the Collaboration on the Theoretical Foundations of Deep Learning (\url{https://deepfoundations.ai/}) through awards DMS-2031883 and \#814639  and the TILOS institute (NSF CCF-2112665). 
 This work used the programs (1) XSEDE (Extreme science and engineering discovery environment)  which is supported by NSF grant numbers ACI-1548562, and (2) ACCESS (Advanced cyberinfrastructure coordination ecosystem: services \& support) which is supported by NSF grants numbers \#2138259, \#2138286, \#2138307, \#2137603, and \#2138296. Specifically, we used the resources from SDSC Expanse GPU compute nodes, and NCSA Delta system, via allocations TG-CIS220009. Research of D.D. was supported by NSF CCF-2023166 and DMS-2306322 awards.

\bibliographystyle{abbrv}
\bibliography{references}

\appendix
\newpage

\section{Lin-RFM as re-parameterization of IRLS}
\label{appendix: Reformulation of lin-RFM as IRLS}

We now show that when $\phi$ is a matrix power $\alpha$, lin-RFM is an inverse-free re-parameterization of IRLS-$p$ with $p = 2 - 4\alpha$.  We recall the IRLS-$p$ algorithm from~\cite{MatrixCompletionIterativeLeastSquaresMohan} below.

\begin{tcolorbox}
\begin{alg}[IRLS-$p$~\cite{MatrixCompletionIterativeLeastSquaresMohan}] Let $f(A) =  \langle A, X \rangle$ and let $\{(A_i, y_i)\}_{i=1}^{n} \subset \mathbb{R}^{d_1 \times d_2} \times \mathbb{R}$ denote training data.  Let $P_0 = I$, select $\epsilon_0 > 0$, and for $t \in \mathbb{Z}_+$,
\label{alg: IRLS-p}
\begin{align*}
    & \textit{Step 1:} \qquad X_t = \argmin_{X \in \mathbb{R}^{d_1 \times d_2}}~ \Tr(P_t X^T X) ~~ \text{subject to $ \langle A_i, X \rangle = y_i$ for all $i \in [n]$}~; \\
    & \textit{Step 2:} \qquad P_{t+1} = \left( X_t^T X_t + \epsilon_t I \right)^{\frac{p}{2} - 1}~ \\
    & \textit{Step 3:} \qquad \text{Select $\epsilon_{t+1} \in (0, \epsilon_t]$.}
\end{align*} 
\end{alg}
\end{tcolorbox}

It is apparent that IRLS-$p$ is a re-parameterization of lin-RFM upon making the following substitution using the notation from Algorithm~\ref{alg: RFM matrix completion}: $X_t = W_t M_t$ and $P_t = M_t^{-2}$.  Note that for $p \in [0, 1]$, our re-parameterization avoids the inverse used in Step 2 of Algorithm~\ref{alg: IRLS-p}.

\section{Lin-RFM for sparse linear regression}
\label{appendix: lin-RFM linear reg}

Below, we state the lin-RFM algorithm in the context of sparse linear regression.  We let $X = [x_1 \ldots x_n]^T \in \mathbb{R}^{n \times d}$ denote training samples (also known as the design matrix) and $y = [y_1 \ldots y_n]^T$ denote the labels.

\begin{tcolorbox}
\begin{alg}[Lin-RFM for sparse linear regression] Let $f(x) = x^T (w \odot m)$ where $w, m \in \mathbb{R}^{d}$ for $i \in [L]$ where $\odot$ denotes elementwise multiplication (Hadamard product).  Let $m_0 = \mathbf{1}$ and for $t \in \mathbb{Z}_+$,
\begin{align*}
    &\textit{Step 1:} \qquad w_t = \argmin_{w \in \mathbb{R}^{d}} \|w\|_2^2 ~\text{subject to $X^T (w \odot m_t) = y$};\\
    &\textit{Step 2:} \qquad \text{Let $m_{t+1} = \phi_i\left( w_t^2 \odot m_t^2\right)$};
\end{align*}
where $\phi_i$ and all powers are applied element-wise.  
\end{alg} 
\end{tcolorbox}
While $f$ takes the form of a linear diagonal neural network considered in~\cite{LinearDiagonalNetKernelRichLazy, LinearDiagonalNetsStepSize}, we use the AGOP instead of using gradient descent to update $w, m$.

\section{Proof of balancedness from\texorpdfstring{~\cite{Balancedness}}{[Balancedness]}}
\label{sec: proof of balancedness}
Here, we restate and simplify the argument from~\cite[Appendix A.1]{Balancedness} for the fact that balancedness is preserved when training deep linear networks using gradient flow.  

\begin{theorem*}[Balancedness from~\cite{Balancedness}]
    Let $f(A) =  \langle A, W^{(L)} W^{(L-1)} \ldots W^{(1)}\rangle$ denote an $L$-layer linear network.  Let $f_t$ denote the network trained for time $t$ using gradient flow to minimize mean squared error on data $\{(A_p, y_p)\}_{p=1}^{n}$. If $\{W_0^{(\ell)}\}_{\ell=1}^{L-1}$ are balanced, i.e., ${W_0^{(\ell)}} {W_0^{(\ell)}}^T = {W_0^{(\ell+1)}}^T {W_0^{(\ell+1)}} $, then for any $t > 0$, $\{W_t^{(\ell)}\}_{\ell=1}^{L-1}$ are balanced. 
\end{theorem*}

\begin{proof}
    The gradient flow ODE for the mean square error reads as:
    \begin{align*}
        \frac{dW^{(\ell)}}{dt} = \sum_{p=1}^{n} (y_{p} - f_t(A_p^T))  {W_t^{(\ell+1)}}^T \ldots {W_t^{(L)}}^T A_p {W_t^{(1)}}^T \ldots {W_t^{(\ell-1)}}^T~.
    \end{align*}
    Multiplying the left hand side of the equation above by ${W_t^{(\ell)}}^T$, we find the following equality holds for all $\ell \in [L]$: 
    \begin{align*}
        \frac{dW^{(\ell)}}{dt} {W_t^{(\ell)}}^T = {W_t^{(\ell+1)}}^T \frac{dW^{(\ell+1)}}{dt}~.
    \end{align*}
    Now upon adding the transpose of the above equation to both sides and using the product rule, we learn
    $$\frac{d}{dt}[{W_t^{(\ell)}} {W_t^{(\ell)}}^T - {W_t^{(\ell+1)}}^T {W_t^{(\ell+1)}}]=0\qquad \forall t\geq 0.$$
    Thus, if we initialized such that ${W_0^{(\ell)}}  {W_0^{(\ell)}}^T = {W_0^{(\ell+1)}}^T  {W_0^{(\ell+1)}}$, then ${W_t^{(\ell)}} {W_t^{(\ell)}}^T = {W_t^{(\ell+1)}}^T  {W_t^{(\ell+1)}}$ for any time $t \geq 0$, which concludes the proof.  
\end{proof}

\section{Proof of Theorem~\ref{thm: deep lin-RFM fixed point analysis}}
\label{appendix: proof of deep lin RFM}

In this section, we prove Theorem~\ref{thm: deep lin-RFM fixed point analysis}, which is restated below for the reader's convenience.

\begin{theorem*}
    Let $\phi_{\ell}(s) = (s + \epsilon)^{\alpha_{\ell}}$ for $\epsilon > 0$ and let $C_{\ell} = \sum_{i=1}^{\ell} \prod_{j=1}^{i-1} \alpha_i (1 - 2\alpha_j)$.  Assume $C_{\ell} > 0$ for all $\ell \in [L]$ and that $\sum_{\ell=1}^{L} C_{\ell} < \frac{1}{2}$.  Then, the fixed points of Algorithm~\ref{alg: deep lin RFM matrix completion} are the first-order critical points of the following optimization problem: 
    \begin{align*}
        \argmin_{Z \in \mathbb{R}^{d_1 \times d_2}} \sum_{j=1}^{d_1} \psi_{\epsilon}(\sigma_j(Z)) ~\text{subject to $\mathcal{A}(Z) = y$}~;
    \end{align*}    
    where $\psi_{\epsilon}(r) := \int_{0}^{r} s h_1(s) \ldots h_L(s)\, ds$ and $h_i(s)$ are defined recursively as 
    \begin{align}
        h_1(s) = (s^2 + \epsilon)^{\alpha_1};\qquad  h_{\ell}(s) = \left( s^2 h_1(s)^{-2} \ldots h_{\ell-1}(s)^{-2} + \epsilon \right)^{\alpha_{\ell}} ~ \text{for $\ell \in \{2, \ldots, L\}$}.
    \end{align}
  Moreover, $\lim\limits_{\epsilon \to 0} \psi_{\epsilon}(r) = r^{2 - 4C_1 - \ldots -4C_L}$ for every $r$.  
\end{theorem*}
\begin{proof}
    As in the analysis presented in Section~\ref{sec: fixed point equation of lin-RFM}, let $Z_t = W_t M_t^{(L)} \ldots M_t^{(1)}$.  As the $\phi_{\ell}$ are positive, the matrices $M_t^{(\ell)}$ are all invertible.  These matrices are also symmetric as the AGOP returns a symmetric matrix. Let $Z, W, M^{(L)}, \ldots, M^{(1)}$ denote the value of these parameters at a fixed point.  At the fixed point, $M^{(\ell)}$ can be recursively written as the following function of $Z^TZ$ and $\epsilon$:  
    \begin{align*}
        M^{(\ell)} = \left({M^{(\ell-1)}}^{-1} \ldots {M^{(1)}}^{-1} Z^T Z {M^{(1)}}^{-1} \ldots {M^{(\ell-1)}}^{-1} \right)^{\alpha_{\ell}} \text{with} ~ M^{(1)} = \left(Z^TZ + \epsilon I\right)^{\alpha_1}~.
    \end{align*}
    As a result, the $M^{(\ell)}$ all commute.  Letting $\mathcal{A}^*$ denote the adjoint of $\mathcal{A}$, the first-order optimality conditions of deep lin-RFM are given by
    \begin{align}
    \label{eq: first order optimality condition deep lin RFM}
        2Z {M^{(1)}}^{-2} \ldots {M^{(L)}}^{-2} + \mathcal{A}^*(\lambda) &= 0 \\
        \label{eq: constraints deep lin RFM}
        \mathcal{A}(Z) &= y~. 
    \end{align}
    The singular values of $Z {M^{(1)}}^{-2} \ldots {M^{(L)}}^{-2}$ can be written as a univariate function of $\epsilon$, which we denote $\kappa_{\epsilon}: [0, \infty) \to [0, \infty)$.  Namely, 
    \begin{align*}
        \kappa_{\epsilon}(x) = x h_1(x) \ldots h_L(x) ~~ ; ~~ h_{\ell}(x) = \left( x^2 h_1(x)^{-2} \ldots h_{\ell-1}(x)^{-2} + \epsilon \right)^{\alpha_{\ell}} ~ \text{for $\ell \in \{2, \ldots, L\}$}~;
    \end{align*}
    where $h_1(x) = (x^2 + \epsilon)^{\alpha_1}$.  Note that $\kappa_{\epsilon}(x)$ is integrable as it is continuous and bounded by $g(x) = x^{1 - 4C_1 - \ldots - 4C_{L}}$, which is integrable as $\sum_{\ell=1}^{L} C_{\ell} < \frac{1}{2}$.   Thus, by the result of~\cite[Theorem 1.1]{LewisSpectralDerivatives}, the fixed points of Algorithm~\ref{alg: deep lin RFM matrix completion} are the first order critical points of the objective
    \begin{align*}
        \argmin_{Z \in \mathbb{R}^{d_1 \times d_2}} \sum_{j=1}^{d_1} \psi_{\epsilon}(\sigma_j(Z)) ~ \text{subject to $\mathcal{A}(Z) = y$}~;
    \end{align*}    
    where $\psi_{\epsilon}(x) =  \int_{0}^{x} \kappa_{\epsilon}(s) ds$.  Next, we use the Monotone Convergence Theorem~\cite[3.11]{axler2023measure} to prove that $\psi_{\epsilon}(x)$ converges to $\int_{0}^{x} g(s) ds$ as $\epsilon \to 0$.  In particular, for any decreasing sequence $\{\epsilon_n\}_{n=1}^{\infty}$ such that $\lim\limits_{n \to \infty} \epsilon_n = 0$, we have $\{\kappa_{\epsilon_n}\}_{n=0}^{\infty}$ are an increasing sequence of non-negative functions with $\lim\limits_{n \to \infty} \kappa_{\epsilon_n}(x) = g(x)$.  Hence, the Monotone Convergence Theorem implies that 
    \begin{align*}
        \lim_{n \to \infty} \psi_{\epsilon_n}(x) = \lim_{n \to \infty} \int_{0}^{x} \kappa_{\epsilon_n}(s) \, ds = \int_{0}^{x} g(s)\, ds = x^{2 - 4C_1 - \ldots - 4C_L}~.
    \end{align*}
    Thus, as $\epsilon_n \to 0$, we have $\lim\limits_{\epsilon_n \to 0} \psi_{\epsilon_n}(x) = x^{2 - 4C_1 - \ldots - 4C_L}$, which concludes the proof. 
\end{proof}

\section{Experimental details}
\label{appendix: Training and hyper-parameter details}

We now outline all training and hyper-parameter tuning details for all experiments.  

\paragraph{Experiments for sparse linear regression.}  We now outline the details for experiments in Fig.~\ref{fig: NoiseLess linear reg}.  We trained lin-RFM models for $1000$ iterations and utilized a ridge-regularization parameter of $10^{-10}$ to avoid numerical issues with linear system solvers from NumPy~\cite{numpy2}.  For deep linear diagonal networks, all initial weights were independently and identically distributed (i.i.d.) according to $\mathcal{N}(0, 10^{-10})$.  Networks were trained using gradient descent with a learning rate of $10^{-1}$ for $10^5$ steps.  To reduce computational burden, models were stopped early if they reached below $10^{-3}$ test error.  For minimizing $\ell_1$-norm, we utilized the CVXPy library~\cite{diamond2016cvxpy, agrawal2018rewriting}.  

\paragraph{Experiments for low rank matrix completion.}  We first outline the details for experiments in Figs.~\ref{fig: NoiseLess matrix completion} and~\ref{fig: Matrix completion scaling}.  For lin-RFM, we grid searched over regularization parameters in $\{5 \times 10^{-2}, 3 \times 10^{-2}, 10^{-2}, 5 \times 10^{-3}, 10^{-3}, 5 \times 10^{-4}, 10^{-4}\}$ and trained for up to $10000$ iterations.  We trained deep linear networks using the same width,  initialization scheme, and optimization method used in~\cite{DeepMatrixFactorization}.  Namely, the work~\cite{DeepMatrixFactorization} used networks of width $d$ for $d \times d$ matrices, initialized weights i.i.d. according to $\mathcal{N}(0, \beta)$ where $\beta = (0.001)^{\frac{1}{L}} d^{-.5}$,  and trained using RMSProp~\cite{RmsProp} with a learning rate of $0.001$.  To reduce computational burden, models were stopped early if they reached below $10^{-3}$ test error.  For minimizing nuclear-norm, we utilized the CVXPy library~\cite{diamond2016cvxpy, agrawal2018rewriting}.

\section{Analysis of lin-RFM for special cases of matrix completion}
\label{sec: lin-RFM matrix completion convergence}

In this section, we establish the following results for lin-RFM when $\phi$ is the identity map:
\begin{enumerate}
    \item We prove that lin-RFM can recover low rank matrices in certain cases when nuclear norm cannot. 
    \item We provide a full convergence analysis for lin-RFM for matrices of size $m \times 2$ in which the first row and one of the columns are observed. 
    \item We prove that low rank solutions are fixed point attractors for lin-RFM on matrices of size $m \times 2$.  
\end{enumerate}
While convergence of IRLS implies convergence of lin-RFM when $\phi(s) = s^{\alpha}$ with $\alpha \in \left[ \frac{1}{4}, \frac{1}{2} \right]$, our convergence analysis for the special cases above are for $\alpha = 1$ corresponding to negative values of $p$ in IRLS-$p$, which to the best of our knowledge have not been previously considered. Moreover, in contrast to previous convergence analyses~\cite{MatrixCompletionIterativeLeastSquaresMohan}, our results in the special settings establish local convergence to the low rank solution and quantify the full range of initializations that result in low rank completions. Recall that the goal of matrix completion is to recover a low rank matrix $Y_{\sharp}$ from pairs of observations $\{(M_{ij}, {Y_{\sharp}}_{ij})\}_{(i,j) \in \Omega}$ where $M_{ij}$ are indicator matrices, ${Y_{\sharp}}_{ij}$ denotes the $(i, j)$ entry in $Y_{\sharp}$, and $\Omega$ is a set of observed coordinates.  For our analyses, we will also modify Algorithm~\ref{alg: RFM matrix completion} to enforce that $(W_t M_t)_{\Omega} = {Y_{\sharp}}_{\Omega}$ for all $t$, where for any matrix $A$, $A_{\Omega}$ denotes the vector of entries of $A$ corresponding to coordinates in $\Omega$.  We begin with the following lemma, which presents a class of examples for which $\alpha = 1$ leads to fixed points that recover low rank matrices.

\begin{lemma}
    \label{lemma: RFM alpha 1}
    Let $Y_{\sharp} \in \mathbb{R}^{d \times d}$ have rank $r$ with each column containing at most $r-1$ zeros.  Suppose $r$ rows of $Y_{\sharp}$ are observed and all other rows contain $r$ observations.  If Algorithm~\ref{alg: RFM matrix completion} trained on ${Y_{\sharp}}_{\Omega}$ with $\phi$ as the identity map converges, then it recovers $Y_{\sharp}$. 
\end{lemma}
\begin{proof}
     Assume without loss of generality that the first $r$ rows of $Y$ are observed. Let $M_{t, j}$ denote the sub-matrix of $M_t$ consisting of columns whose indices match the column indices of observed entry in row $j$ of $Y_{\Omega}$.  Let $Y_{\Omega(j)}$ denote the vector of observed entries in row $j$ of $Y_{\Omega}$. Then, as the modified Algorithm~\ref{alg: RFM matrix completion} replaces observed entries with their values, 
    \begin{align*}
        M_{t+1} = \sum_{i = 1}^{r} Y_{i,:}^T Y_{i,:} + \sum_{j=r+1}^{d} M_t^T {M_{t, j}^{\dagger}}^T Y_{\Omega(j)}Y_{\Omega(j)}^T M_{t, j}^{\dagger} M_t  ~.
    \end{align*}
    Hence, if Algorithm~\ref{alg: RFM matrix completion} converges, then $M = \lim_{t \to \infty} M_{t}$ is symmetric and satisfies
    \begin{align*}
        M = \sum_{i = 1}^{r} Y_{i,:}^T Y_{i,:}   + \sum_{j=r+1}^{d} M {M_{j}^{\dagger}}^T Y_{\Omega(j)}Y_{\Omega(j)}^T M_{j}^{\dagger} M ~ \implies M  - MQM  =  R~;
    \end{align*}    
    where $Q = \sum_{j=r+1}^{d} {M_{j}^{\dagger}}^T Y_{\Omega(j)}Y_{\Omega(j)}^T M_{j}^{\dagger} $, and $R = \sum_{i = 1}^{r} Y_{i,:}^T Y_{i,:} $.  Note $R$ is of rank $r$ since $Y$ is of rank $r$.  Assume $M$ has rank $q > r$.  Let $M = \sum_{i=1}^{q} \sigma_i v_i v_i^T$ given by the singular value decomposition.  Now observe that $Q$ has row space and column space contained in the span of $\{v_i\}_{i=1}^{q}$ since the column space of $M_j$ is contained in the column space of $M$.  Now as $M$ and $MQM$ have the same left and right singular vectors, $Q$ must have rank $q-r$ since otherwise $M - MQM$ cannot be of rank $r$.  Thus, we conclude that the column vectors $\{M_j\}_{j=r+1}^{d}$ must lie in the span of $q-r$ elements in $\{v_i\}_{i=1}^{q}$.  Without loss of generality we can assume these elements are $\{v_i\}_{i=1}^{q-r}$.  Thus, there exist at least $r$ columns of $M$ such that $v_j^T M$ is zero in each of these columns for $j \in \{q-r+1, \ldots, q\}$. On the other hand, note that $v_j^T M = \sigma_j v_j^T$ by the singular value decomposition of $M$.  Hence, we conclude that there exist at $r$ columns such that $v_j^T$ is zero in each of these columns for all $j \in \{q-r+1, \ldots, q\}$.  As $M - MQM = R$, $R = \sum_{i=q-r+1}^{q} \sigma_i(R) v_i v_i^T$ where $\{\sigma_i(R)\}_{i=q-r+1}^{q}$ denotes the singular values of $r$ and so, we conclude that there exist at least $r$ coordinates $A = \{a_i\}_{i=1}^r \subset [d]$ such that $R$ contains zeros in rows and columns indexed by $A$.  Hence, there exist at least $r$ diagonal entries of $R$ that are zero.  As the $j$\textsuperscript{th} diagonal entry of $R$ is given by $\sum_{i=1}^{r} Y_{ij}^2$, this implies that there is at least one column of $Y$ consisting of $r$ zeros, which is a contradiction since we assumed $Y$ has at most $r-1$ zeros per column.  
\end{proof}

\noindent We now provide a simple example demonstrating that minimum nuclear norm will not necessarily recover a rank $r$ matrix in the setting of Lemma~\ref{lemma: RFM alpha 1}.

\begin{example}
Let $Y_{\sharp} = \begin{bmatrix} 1 & 1 & 1 \\ 1 & 1 & 1 \\ 1 & 1 & 1 \end{bmatrix}$ and let $\tilde{Y} = \begin{bmatrix} 1 & 1 & 1 \\ 1 & ? & ? \\ 1 & ? & ?  \end{bmatrix}$~.  The minimum nuclear norm solution is given by the rank 3 matrix $Y_{nn} = \begin{bmatrix} 1 & 1 & 1 \\ 1 & \frac{1}{2} & \frac{1}{2} \\ 1 & \frac{1}{2} & \frac{1}{2}  \end{bmatrix}$ with $\| Y_{nn}\|_* = \sqrt{3 - 2\sqrt{2}} + \sqrt{3 + 2\sqrt{2}} \approx 2.83$.  In contrast, Lemma~\ref{lemma: RFM alpha 1} implies that if Algorithm~\ref{alg: RFM matrix completion} converges, it will recover $Y$ exactly from $\tilde{Y}$.  
\end{example}

Next, we analyze the case of matrices of size $m \times 2$ with nonzero entries and prove that lin-RFM with $\phi$ as the identity map converges to the rank $1$ solution for a large range of initializations $M_0$, as is shown by the following proposition.

\begin{prop}
\label{prop: convergence}
    Let $Y \in \mathbb{R}^{m \times 2}$ such that $Y_{11} = a, Y_{12} =  b$ are observed and that $Y_{i2}$ are observed for $i \in [m]$.  Let $s$ denote the sum of squares of observed entries in rows 2 through $m$ and assume $a, b, s \neq 0$ with $\frac{b}{a} > 0$.  Let $\hat{Y}(t)$ denote the solution after $t$ steps of Algorithm~\ref{alg: RFM matrix completion} with $L = 2$, $\phi$ as the identity map, and $M_0 = \begin{bmatrix} u & \epsilon \\ \epsilon & v\end{bmatrix}$.  As $t \to \infty$, Algorithm~\ref{alg: RFM matrix completion} converges to $\hat{Y}$ that is rank $1$ if and only if  $\frac{u\epsilon + v\epsilon}{\epsilon^2 + v^2} > \frac{-b(a^2 +b^2 + s)}{as}$.  
\end{prop}

\begin{proof}
    We first prove the statement for $d = 2$ assuming $Y_{22} = c$.  As in the proof of Lemma~\ref{lemma: RFM alpha 1},  we utilize the following form of the updates for $M_t$: 
    \begin{align*}
            M_{t+1} = Y_{1:}^T Y_{1:} + M_{t}^T {M_{t, 2}^{\dagger}}^T Y_{22}^2 M_{t, 2}^{\dagger} M_t~;
    \end{align*}
    where $Y_{1:}$ denotes the first row of $Y$ and $M_{t, 2}$ denotes the second column of $M_t$.  We will show by induction that $M_{t}$ for $t \geq 1$ has the form: 
    \begin{align*}
        M_t = \begin{bmatrix} a^2 + k_t^2c^2  & ab + k_tc^2  \\ ab + k_tc^2 & b^2  + c^2\end{bmatrix} ~. 
    \end{align*}
    We begin with the base case for $t = 1$.  In particular, we have: 
    \begin{align*}
        M_1 &= \begin{bmatrix} a^2   & ab \\ ab & b^2\end{bmatrix} +  \frac{c^2}{(\epsilon^2 + v^2)^2} \begin{bmatrix} u & \epsilon \\ \epsilon & v\end{bmatrix}  \begin{bmatrix}  \epsilon \\ v\end{bmatrix}  \begin{bmatrix}  \epsilon & v\end{bmatrix}  \begin{bmatrix} u & \epsilon \\ \epsilon & v\end{bmatrix} \\
        &= \begin{bmatrix} a^2  + \frac{(u\epsilon + v\epsilon)^2 c^2}{(\epsilon^2 + v^2)^2} & ab + \frac{(u\epsilon + v\epsilon) c^2}{\epsilon^2 + v^2} \\ ab + \frac{(u\epsilon + v\epsilon)c^2}{\epsilon^2 + v^2}  & b^2  + c^2 \end{bmatrix}~.
    \end{align*}
    Hence, the statement holds for $k_1 = \frac{u\epsilon + v\epsilon}{\epsilon^2 + v^2}$ provided $u\epsilon + v\epsilon \neq 0$. Now we assume that the statement holds for $t$ and prove it holds for iteration $t+1$.  Namely, we have: 
    \begin{align*}
         M_{t+1} &= \begin{bmatrix} a^2   & ab \\ ab & b^2\end{bmatrix} +  \frac{c^2}{(k_tc^2 + ab)^2 + (c^2 + b^2)^2} \begin{bmatrix} a^2 + k_t^2c^2 & ab + k_tc^2 \\ ab + k_tc^2  & b^2 + c^2\end{bmatrix}  \begin{bmatrix}  ab + k_tc^2 \\ b^2 + c^2\end{bmatrix}  \\
         & \hspace{61mm} \cdot \begin{bmatrix}  ab + k_tc^2 & b^2 + c^2\end{bmatrix} \begin{bmatrix} a^2 + k_t^2 c^2 & ab + k_tc^2 \\ ab + k_tc^2  & b^2 + c^2\end{bmatrix} \\
         &= \begin{bmatrix} a^2  + \frac{c^2(a^2 + b^2 + c^2 + k_t^2c^2)^2 (ab + k_tc^2)^2}{(ab + k_tc^2)^2 + (c^2 + b^2)^2} &  ab + \frac{(a^2 + b^2 + c^2 + k_t^2c^2) (ab + k_tc^2)}{(ab + k_tc^2)^2 + (c^2 + b^2)} \\ ab + \frac{(a^2 + b^2 + c^2 + k_t^2c^2)^2 (ab + k_tc^2)^2}{(ab + k_tc^2)^2 + (c^2 + b^2)^2} & b^2 + c^2 \end{bmatrix}~.
    \end{align*}
    Hence, the statement holds for $k_{t+1} = \frac{(a^2 + b^2 + c^2 + k_t^2c^2) (ab + k_tc^2)}{(ab + k_tc^2)^2 + (c^2 + b^2)^2}$.  Thus, induction is complete.  Next, we analyze the fixed points of the iteration for $k_{t}$.  In particular, letting $x = \lim_{t \to \infty} k_t$, we have
    \begin{align*}
        x = \frac{(a^2 + b^2 + c^2 + x^2c^2) (ab + xc^2)}{(ab + xc^2)^2 + (c^2 + b^2)^2} \implies (ac^2x + b(a^2 + b^2 + c^2))(bx - a) = 0
    \end{align*}
    for real values of $a, b$.  Hence there are two finite fixed points corresponding to $x = \frac{a}{b}$ and $x = \frac{-b(a^2 +b^2 + c^2)}{ac^2}$ and fixed points at $\pm \infty$.  Next, we analyze the stability of each of these fixed points by analyzing the derivative of $f(x) = \frac{(a^2 + b^2 + c^2 + x^2 c^2) (ab + xc^2)}{(ab + xc^2)^2 + (c^2 + b^2))}$ at each of the fixed points.  In particular, we find that: 
    \begin{align*}
        &f'\left( \frac{a}{b} \right) = \frac{c^2}{b^2 + c^2} ~; ~ f'\left(\frac{-b(a^2 +b^2 + c^2)}{ac^2} \right) = \frac{a^2 + b^2 + c^2}{b^2 + c^2} ~;~ \lim_{x \to \infty} f'(x) = 1 ~;~ \lim_{x \to -\infty} f'(x) = 1 ~.
    \end{align*}
    Hence, we conclude that $x = \frac{a}{b}$ is an attractor, $x =\frac{-b(a^2 +b^2 + c^2)}{ac^2}$ is a repeller, and $x = \pm \infty$ are neutral fixed points.  Next, by computing the regions for which $f(x) < x$ and $f(x) > x$ and noting that $f(x) > 0$ for $x \geq 0$, we conclude that the basin of attraction for $\frac{a}{b}$ is $x \in \left(\frac{-b(a^2 +b^2 + c^2)}{ac^2},  \infty \right)$.  If $k_1 > \frac{-b(a^2 +b^2 + c^2)}{ac^2}$, then $\lim\limits_{t \to \infty} k_t = \frac{a}{b}$, otherwise $\lim\limits_{t \to \infty} k_t = -\infty$.  To conclude the proof for the case of $d = 2$, we show that $\lim_{t \to \infty} \hat{Y}(t)_{21} = c \frac{a}{b}$ when $k_1 > \frac{-b(a^2 +b^2 + c^2)}{ac^2}$.   Next, for any finite $t$, we compute the entry $Y_{21}$ as a function of $k_t$.  Recalling that our solution is of the form $\hat{Y}(t) = W_t M_t$, we have: 
    \begin{align*}
       & \begin{bmatrix}  W_{t, 21} & W_{t, 22} \end{bmatrix} =  \begin{bmatrix}  \frac{c(a b + k_tc^2)}{(a b + k_tc^2)^2 + (b^2 + c^2)^2}  & \frac{c(b^2 + c^2)}{(a b + k_tc^2)^2 + (b^2 + c^2)^2} \end{bmatrix} \\
       &\implies \hat{Y}(t)_{21} =  \frac{c(a^2 + b^2 + c^2 + k_t^2c^2) (ab + k_tc^2)}{(ab + k_tc^2)^2 + (c^2 + b^2)^2} = cf(k_t)~.
    \end{align*}
    Since we initialize such that $k_1 > \frac{-b(a^2 +b^2 + c^2)}{ac^2}$,  $\lim\limits_{t \to \infty} k_t = \frac{a}{b}$ and so, $\lim\limits_{t \to \infty} \hat{Y}(t)_{21} = c \frac{a}{b}$, which concludes the proof for $d = 2$.  Now, suppose that $d > 2$ and without loss of generality that entries in column $2$ of $Y$ are all observed.  Then, we have
    \begin{align*}
            M_{t+1} = Y_{1:}^T Y_{1:} + M_{t}^T {M_{t, 2}^{\dagger}}^T \left(\sum_{i=2}^{d} Y_{i2}^2 \right) M_{t, 2}^{\dagger} M_t =   Y_{1:}^T Y_{1:} + M_{t}^T {M_{t, 2}^{\dagger}}^T  M_{t, 2}^{\dagger} M_t s~.
    \end{align*}
    Thus, the result holds upon replacing $c^2$ with $s$, which concludes the proof.
\end{proof}

Note that we do not require any regularization to avoid singularities in the objective, which as shown in Corollary~\ref{corollary: Fixed point matrix powers} corresponds to a negative hyperbolic objective.  The following proposition extends the analysis of the Proposition~\ref{prop: convergence} to the setting where either column can contain observed entries.

\begin{figure}[!t]
    \centering
    \includegraphics[width=\textwidth]{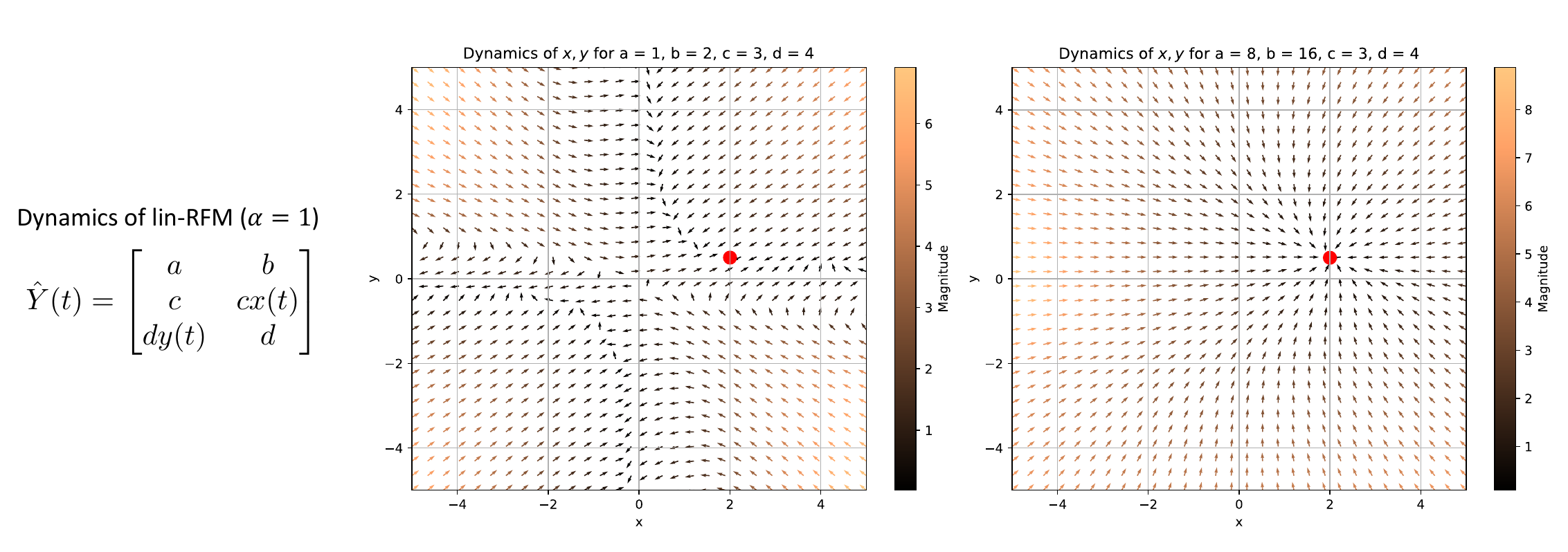}
    \caption{Examples of basins of attraction for lin-RFM with $\phi$ as the identity map for matrices $Y \in \mathbb{R}^{3 \times 2}$ with $Y_{11} = a, Y_{12} = b, Y_{21} = c, Y_{32} = d$ with $a, b, c, d > 0$.  Such matrices are a subset of those analyzed in Proposition~\ref{prop: convergence any obs} for which we proved the rank 1 solution is a fixed point attractor.  In this setting, lin-RFM is governed by a two dimensional dynamical system involving variables $x(t), y(t)$.  We plot vector field $f(x, y)$ governing the evolution of $x(t), y(t)$ for various values of $a, b, c, d$ and plot in red the attractor $\mathbf{u} := \left(\frac{b}{a} , \frac{a}{b}\right)$ corresponding to rank 1 solution.  Note that $x(0) = y(0) = 0$ when $M_0 = I$ in lin-RFM.  The strength of the attractor is given by the maximum eigenvalue of the Jacobian of $f$ at $\mathbf{u}$, which is $\frac{a^2 d^2 + b^2 c^2}{a^2 b^2 + a^2 d^2 + b^2 c^2}$.  This eigenvalue grows smaller when $a, b$ increase and leads to a stronger attractor for the rank 1 solution. This behavior is shown in the right hand plot where we scale the values of $a, b$ without altering $c, d$.}
    \label{fig: Attractor Analysis}
\end{figure}

\begin{prop}
\label{prop: convergence any obs}
    Let $Y \in \mathbb{R}^{m \times 2}$ such that $Y_{11} = a, Y_{12} =  b$ with $a, b > 0$ are observed and that one entry in each of the remaining rows is observed.  The unique rank 1 reconstruction for $Y$ is a fixed point attractor of Algorithm~\ref{alg: RFM matrix completion}.   
\end{prop}

\begin{proof}
    The proof is an extension of the proof of Proposition~\ref{prop: convergence}.  We again start by proving the statement for $m = 3$ and assuming the following observation structure in $Y$:
    \begin{align*}
        Y =     \begin{bmatrix}  a & b \\ c & ? \\ ? & d \end{bmatrix}~.        
    \end{align*}
    Note that $M_t \in \mathbb{R}^{2 \times 2}$ in this setting.  For the observation pattern above Algorithm~\ref{alg: RFM matrix completion} proceeds to update $M_t$ as follows:
    \begin{align*}
            M_{t+1} = Y_{1:}^T Y_{1:} + M_{t}^T {M_{t, 1}^{\dagger}}^T Y_{21}^2 M_{t, 1}^{\dagger} M_t + M_{t}^T {M_{t, 2}^{\dagger}}^T Y_{32}^2 M_{t, 2}^{\dagger} M_t~;
    \end{align*}   
    where $Y_{1:}$ denotes the first row of $Y$ and $M_{t, i}$ denotes column $i$ of $M_t$.  We will now prove using induction that 
    \begin{align*}
        M_t = \begin{bmatrix} a^2 + c^2 + y_t^2 d^2 & ab + x_t c^2  + y_t d^2  \\ ab + x_t c^2  + y_t d^2 & b^2 + d^2 + x_t^2 c^2\end{bmatrix}~;
    \end{align*}
    with 
    \begin{align*}
        x_{t+1} &= \frac{(a^2 + b^2 + c^2 + d^2 + x_t^2 c^2 + y_t^2 d^2) (ab + x_t c^2 + y_t d^2)}{(ab + x_t c^2 + y_t d^2)^2 + (a^2 + c^2 + y_t^2 d^2)^2 } ; \\
        y_{t+1} &= \frac{(a^2 + b^2 + c^2 + d^2 + x_t^2 c^2 + y_t^2 d^2) (ab + x_t c^2 + y_t d^2)}{(ab + x_t c^2 + y_t d^2)^2 + (b^2 + d^2 + x_t^2 d^2)^2 }~.       
    \end{align*}
    We begin with the base case of $t = 1$.  For $t = 0$, we start with $M_0 = \begin{bmatrix} u & \epsilon \\ \epsilon & v \end{bmatrix}$.  After one step, by direct computation, we have 
    \begin{align*}
        M_{1} &= \begin{bmatrix} a^2 + c^2 + y_1^2 d^2 & ab + c^2 x_1 + d^2 y_1 \\ ab + c^2 x_1 + d^2 y_1 & b^2 + d^2 + x_1^2 c^2 \end{bmatrix} ~~;~~
        x_1 = \frac{(u + v) \epsilon}{u^2 + \epsilon^2} ~~;~~ 
        y_1 = \frac{(u + v) \epsilon}{v^2 + \epsilon^2} ~.
    \end{align*}
    Hence, the base case holds.  We now prove the statements for time $t+1$ assuming the form for $M_t$.  To simplify notation let $M_t = \begin{bmatrix} A & B \\ B & C\end{bmatrix} $ where $A := a^2 + c^2 + y_t^2 d^2$, $B := ab + x_t c^2  + y_t d^2$ and $C := b^2 + d^2 + x_t^2 c^2$.  Then, 
    \begin{align*}
         M_{t+1} &= \begin{bmatrix} a^2   & ab \\ ab & b^2\end{bmatrix} +  \frac{c^2}{(A^2 + B^2)^2} \begin{bmatrix} (A^2 + B^2)^2 & (A^2 + B^2) (AB + BC) \\ (A^2 + B^2) (AB + BC) & (AB + BC)^2 \end{bmatrix} \\
         &\hspace{19mm} + \frac{d^2}{(B^2 + C^2)^2} \begin{bmatrix} (AB + BC)^2 & (B^2 + C^2) (AB + BC) \\ (B^2 + C^2) (AB + BC) & (B^2 + C^2)^2 \end{bmatrix}.
    \end{align*}
    Hence, by letting $x_{t+1} = \frac{AB + BC}{A^2 + C^2}$ and $y_{t+1} = \frac{AB + BC}{B^2 + C^2}$ and substituting back the values for $A, B, C$, we conclude
    \begin{align*}
        M_{t+1} &= \begin{bmatrix} a^2  + c^2 + d^2 y_{t+1}^2 & ab + c^2 x_t + d^2 y_t \\ ab + c^2 x_t + d^2 y_t & b^2 + d^2 + c^2 x_{t+1}\end{bmatrix} ; \\
        x_{t+1} &= \frac{(a^2 + b^2 + c^2 + d^2 + x_t^2 c^2 + y_t^2 d^2) (ab + x_t c^2 + y_t d^2)}{(ab + x_t c^2 + y_t d^2)^2 + (a^2 + c^2 + y_t^2 d^2)^2 } ; \\
        y_{t+1} &= \frac{(a^2 + b^2 + c^2 + d^2 + x_t^2 c^2 + y_t^2 d^2) (ab + x_t c^2 + y_t d^2)}{(ab + x_t c^2 + y_t d^2)^2 + (b^2 + d^2 + x_t^2 d^2)^2 }~.       
    \end{align*}
    Hence, the evolution of $M_t$ is governed by the two dimensional dynamics given by the map
    \begin{align*}
        f(x, y) = \Bigg( &\frac{(a^2 + b^2 + c^2 + d^2 + x^2 c^2 + y^2 d^2) (ab + x c^2 + y d^2)}{(ab + x c^2 + y d^2)^2 + (a^2 + c^2 + y^2 d^2)^2 }, \\
        & \frac{(a^2 + b^2 + c^2 + d^2 + x^2 c^2 + y^2 d^2) (ab + x c^2 + y d^2)}{(ab + x c^2 + y d^2)^2 + (b^2 + d^2 + x^2 d^2)^2 } \Bigg).
    \end{align*}    
    By direct evaluation, $f\left(\frac{b}{a} , \frac{a}{b} \right) = \left(\frac{b}{a} , \frac{a}{b} \right) $ and so the point $\mathbf{u} = \left( \frac{b}{a} , \frac{a}{b} \right)$ is a fixed point of $f(x, y)$.  Hence, all that remains is to show that the fixed point is an attractor.  This can be done by showing that the Jacobian of $f$ at the point $\mathbf{u}$ has top eigenvalue less than $1$.  Again, through direct evaluation, the Jacobian of $f$ is rank $1$ at this point, and the top eigenvalue is given by $\frac{a^2 d^2 + b^2 c^2}{a^2 b^2 + a^2 d^2 + b^2 c^2}$, which is less than $1$ for all values of $a, b, c, d$.  Thus, we have proved the proposition for the case when $m = 3$.  As in the proof of Proposition~\ref{prop: convergence}, note that this analysis extends to the case of arbitrary $m$ by replacing $c^2$ with $s_1$ and $d^2$ with $s_2$ where $s_1, s_2$ is the sum of squares of observations in columns $1$ and $2$, respectively, after excluding $a, b$ in the first row. 
\end{proof}

\paragraph{Remarks.} Unlike the case for Proposition~\ref{prop: convergence} where Algorithm~\ref{alg: RFM matrix completion} can be analyzed using a one dimensional dynamical system, the above case results in a more involved two dimensional dynamical system.  Nevertheless, we empirically find that the rank 1 reconstruction again contains a large basin of attraction for many instantiations of $a, b$ and $M_0$, including $M_0 = I$ (see Fig.~\ref{fig: Attractor Analysis}).  Moreover, it remarkably appears that the rank 1 solution is the only finite fixed point attractor for initializations $M_0$ with positive entries. Interestingly, the strength of the attractor corresponding to the low rank solution is governed by the magnitude of observations, with larger values of $a, b$ resulting in stronger attractors.

Lastly, we prove that lin-RFM with $\phi$ as the identity map can result in a completion that drives all norms (and penalties) to infinity when the matrix $Y_{\sharp}$ contains entries that are zero. This behavior matches that of deep linear networks trained using gradient descent~\cite{razin2020implicit} that do not necessarily learn minimum $\ell_2$ norm solutions in weight space.  In particular, we consider the example from~\cite{razin2020implicit} in which the goal is to impute the missing entry (denoted as ?) in the following matrix: 
\begin{align*}
    \tilde{Y} = \begin{bmatrix} 1  & 0 \\ ? & 1 \end{bmatrix}~.
\end{align*}
The following lemma proves that lin-RFM with $\phi$ as the identity map drives this missing entry to infinity.

\begin{lemma}
    \label{lemma: all norms to infty}
    Let $Y \in \mathbb{R}^{2 \times 2}$ such that $Y_{11} = 1, Y_{12} = 0, Y_{22} = 1$.  Let $\hat{Y}(t)$ denote the solution after $t$ steps of Algorithm~\ref{alg: RFM matrix completion} with $L = 2$, $\phi$ as the identity map, and $M_0 = \begin{bmatrix} u & \epsilon \\ \epsilon & v\end{bmatrix}$ satisfying $u\epsilon + v\epsilon \neq 0$.  Then, $\lim\limits_{t \to \infty} \left|\hat{Y}(t)_{21}\right| = \infty$.
\end{lemma}
\begin{proof}
    As in the proof of Lemma~\ref{lemma: RFM alpha 1}, we utilize the following form of the updates for $M_t$: 
    \begin{align*}
        M_{t+1} = Y_{1:}^T Y_{1:} + M_{t}^T {M_{t, 2}^{\dagger}}^T Y_{22}^2 M_{t, 2}^{\dagger} M_t~;
    \end{align*}
    where $Y_{1:}$ denotes the first row of $Y$ and $M_{t, 2}$ denotes the second column of $M_t$.  We will show by induction that $M_{t}$ for $t \geq 1$ has the form: 
    \begin{align*}
        M_t = \begin{bmatrix} 1 + a_t^2  & a_t  \\ a_t & 1\end{bmatrix} ; 
    \end{align*}
    where $\{a_t\}_{t \geq 1}$ is an increasing sequence with $a_1 \neq 0$ and $\lim\limits_{t \to \infty} |a_t| \to \infty$.  We start with the base case for $t = 1$:
    \begin{align*}
        M_1 &= \begin{bmatrix} 1   & 0 \\ 0 & 0\end{bmatrix} +  \frac{1}{(\epsilon^2 + v^2)^2} \begin{bmatrix} u & \epsilon \\ \epsilon & v\end{bmatrix}  \begin{bmatrix}  \epsilon \\ v\end{bmatrix}  \begin{bmatrix}  \epsilon & v\end{bmatrix}  \begin{bmatrix} u & \epsilon \\ \epsilon & v\end{bmatrix} \\
        &= \begin{bmatrix} 1  + \frac{(u\epsilon + v\epsilon)^2}{(\epsilon^2 + v^2)^2} & \frac{u\epsilon + v\epsilon}{\epsilon^2 + v^2} \\ \frac{u\epsilon + v\epsilon}{\epsilon^2 + v^2}  & 1 \end{bmatrix}~.
    \end{align*}
    Hence, the statement holds for $a_1 = \frac{u\epsilon + v\epsilon}{\epsilon^2 + v^2}$ provided $u\epsilon + v\epsilon \neq 0$. Now we assume that the statement holds for $t$ and prove it holds for iteration $t+1$.  Namely, we have: 
    \begin{align*}
         M_{t+1} &= \begin{bmatrix} 1   & 0 \\ 0 & 0\end{bmatrix} +  \frac{1}{(a_t^2 + 1)^2} \begin{bmatrix} 1 + a_t^2 & a_t \\ a_t & 1\end{bmatrix}  \begin{bmatrix}  a_t \\ 1\end{bmatrix}  \begin{bmatrix}  a_t & 1\end{bmatrix} \begin{bmatrix} 1 + a_t^2 & a_t \\ a_t & 1\end{bmatrix} \\
         &= \begin{bmatrix} 1  + \frac{(a_t^3 + 2a_t)^2}{(a_t^2 + 1)^2} &  \frac{(a_t^3 + 2a_t)}{(a_t^2 + 1)} \\ \frac{(a_t^3 + 2a_t)}{(a_t^2 + 1)}  & 1 \end{bmatrix}~.
    \end{align*}
    Hence, the statement holds for $a_{t+1} = \frac{(a_t^3 + 2a_t)}{(a_t^2 + 1)}$.  Thus, induction is complete.  Our proof additionally shows that $\{|a_t|\}$ is a strictly monotonically increasing, diverging sequence since: 
    \begin{align*}
       |a_{t+1}| = \left|\frac{(a_t^3 + 2a_t)}{(a_t^2 + 1)} \right|= |a_t| \left| \frac{(a_t^2 + 2)}{(a_t^2 + 1)}\right| = |a_t| |c_t| ~; 
    \end{align*}
    where $|c_t| > 1$.  Next, for any finite $t$, we compute the entry $Y_{21}$ as a function of $a_t$.  Recalling that our solution is of the form $\hat{Y}(t) = W_t M_t$, we have: 
    \begin{align*}
       \begin{bmatrix}  W_{t, 21} & W_{t, 22} \end{bmatrix} =  \begin{bmatrix} \frac{a_t}{a_t^2 + 1} & \frac{1}{a_t^2 +1} \end{bmatrix} \implies \hat{Y}(t)_{21} = a_t + \frac{a_t}{a_t^2 +1}~.
    \end{align*}
    Hence, we have that $\lim\limits_{t \to \infty} \left|\hat{Y}(t)_{21}\right| = \infty$, which concludes the proof. 
\end{proof}

The work~\cite{razin2020implicit} showed that increasing depth in linear networks initialized near zero and trained using gradient descent to reconstruct $\tilde{Y}_{11}, \tilde{Y}_{12}, \tilde{Y}_{22}$ resulted in reconstructions driving $\tilde{Y}_{21}$ to infinity.  Upon rescaling entries in the second row of $\tilde{Y}$ by the maximum value in the row, such reconstructions would indeed correspond to a rank 1 reconstruction of the above matrix.  Lemma~\ref{lemma: all norms to infty} proves that lin-RFM exhibits this same behavior for this example.

\clearpage

\begin{figure}[!t]
    \centering
    \includegraphics[width=\textwidth]{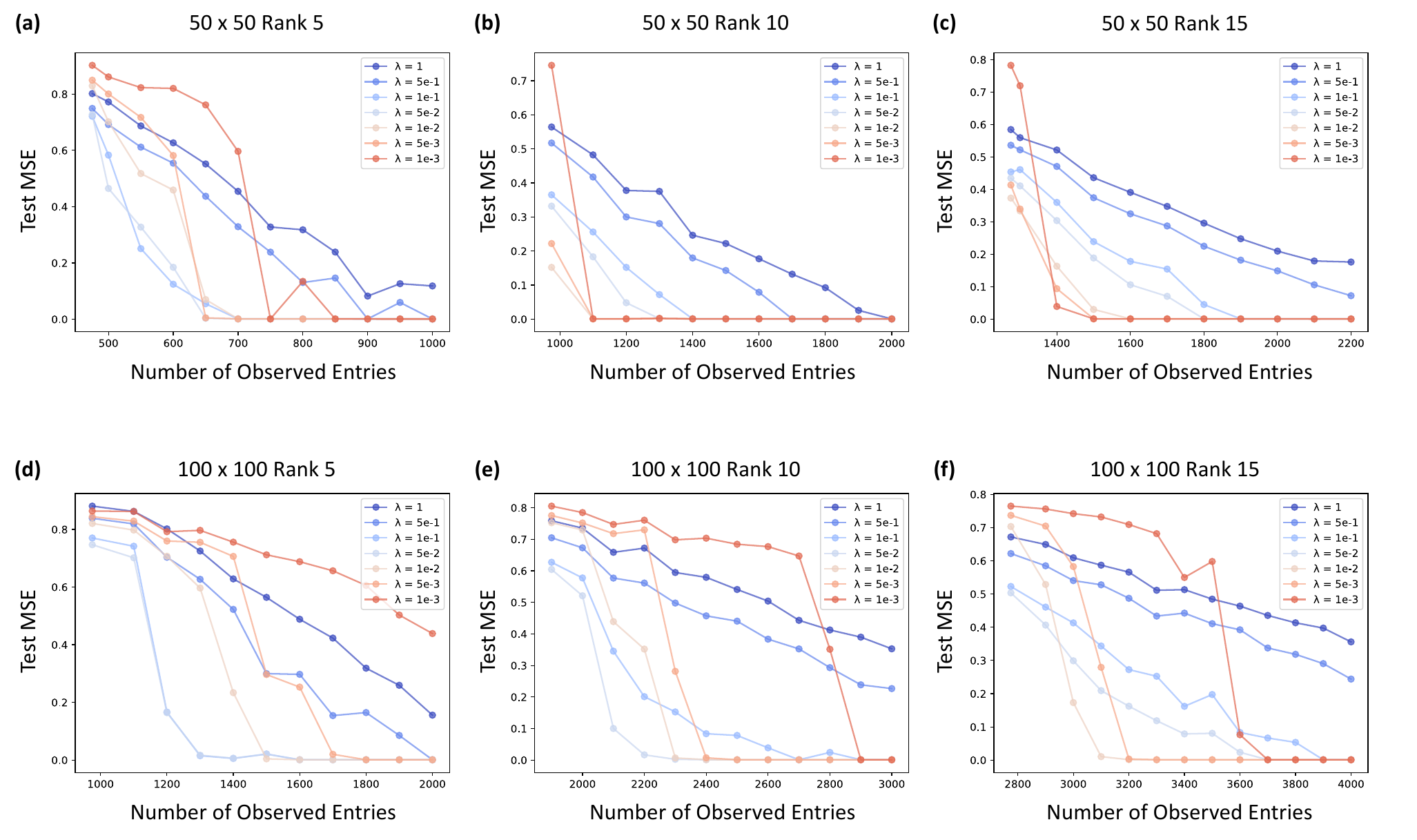}
    \caption{Performance of lin-RFM with $\alpha = 1$ across values of ridge-regularization parameter $\lambda \in \{1, .5, .1, .05, .01, .005, .001\}$.  We generally observe that decreasing ridge-regularization parameter leads to lower test error as rank of the ground truth matrix increases.}
    \label{fig: Regularization}
\end{figure}

\clearpage 

\begin{figure}[!t]
    \centering
    \includegraphics[width=\textwidth]{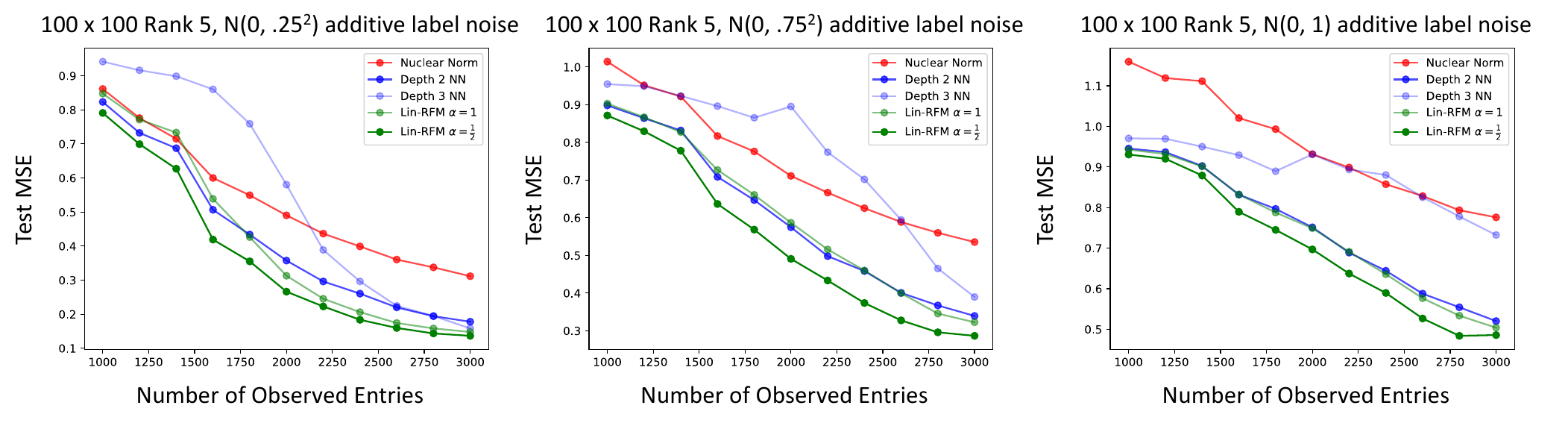}
    \caption{Performance of lin-RFM with $\alpha = 1, \alpha = \frac{1}{2}$, nuclear norm minimization, and depth $2, 3$ linear networks for recovering entries in $100 \times 100$ rank $5$ matrices with i.i.d. label noise $\epsilon \sim \mathcal{N}(0, \sigma^2)$ added to training observations.  Test MSE is calculated using non-noisy data. We consider ridge-regularization values of $\lambda \in \{10, 5, 1, .5, .1, .05, .01, .005, .001\}$ for lin-RFM.  All results are averaged over $5$ draws of data and noise.}
    \label{fig: Label Noise}
\end{figure}

\end{document}